\documentclass[lettersize,journal]{IEEEtran} 

\usepackage[american]{babel}
\usepackage{algorithm}
\usepackage{algpseudocode}
\usepackage{graphicx}
\usepackage{subcaption}
\usepackage{amssymb}
\usepackage{amsthm}

\usepackage{mathtools} 
\usepackage{autonum}
\usepackage{booktabs} 
\usepackage{tikz} 
\usetikzlibrary{patterns}

\usepackage{subfig}
\usepackage{epsfig}
\usepackage{url}
\usepackage{caption}

\usepackage{bm}

\usepackage{color}

\theoremstyle{plain}
\newtheorem{theorem}{Theorem}[section]

\theoremstyle{definition}

\theoremstyle{remark}

\newtheorem{fact}{Fact}
\newtheorem{lemma}{Lemma}
\usepackage{cite}

\begin{document}

\title{Worst-Case Regret Bounds for Combinatorial Thompson Sampling in Sleeping Semi-Bandits}

\author{Zhiming Huang,~\IEEEmembership{Member,~IEEE,} Bingshan Hu,~\IEEEmembership{Member,~IEEE,} Jianping Pan,~\IEEEmembership{Fellow,~IEEE}
\thanks{A preliminary version of this paper has been published in IEEE INFOCOM 2026~\cite{zhiming2026bridging}. In this paper, we study a generalized version of the algorithms studied in the preliminary version, and give tighter upper bounds for swap regret.

Zhiming Huang is with the Paul G. Allen School of Computer Science \& Engineering, University of Washington, WA, USA, Bingshan Hu is with the Department of Computer Science, University of British Columbia, BC, Canada, and Jianping Pan is with the Department of Computer Science, University of Victoria, BC, Canada Emails: zhimingh@cs.uwashington.edu, bingshanhu3@gmail.com, pan@uvic.ca. 
}
}

\markboth{Journal of \LaTeX\ Class Files,~Vol.~14, No.~8, August~2021}%
{Shell \MakeLowercase{\textit{et al.}}: High-probability swap-regret upper bounds for multi-agent bandits.}


\maketitle

\begin{abstract}


We revisit combinatorial Thompson sampling (CTS) for semi-bandits with sleeping arms, where arm availability varies over time and actions must satisfy combinatorial constraints, as in wireless mesh routing with fluctuating link availability. Despite its practical relevance, CTS has been hindered by several long-standing problems: (i) the absence of worst-case regret guarantees in the semi-bandit setting even without sleeping arms, (ii) the lack of theory under adversarially varying availability, and (iii) the consistently weak empirical performance of CTS with Gaussian priors (CTS-G).

This paper resolves these long-standing issues by providing the first worst-case regret analysis of CTS-G, proving an upper bound of $\tilde{O}(m\sqrt{NT})$ and a matching lower bound of $\tilde{\Omega}(m\sqrt{NT})$. To bridge the gap between theory and practice, we further propose CL-SG, a simple CTS-G variant that samples a single shared Gaussian seed each round to coordinate exploration across arms. We show that CL-SG achieves an improved regret bound of $\tilde{O}(\sqrt{mNT})$, together with a matching lower bound $\Omega(\sqrt{mNT})$. Experiments on real-world datasets demonstrate that CL-SG consistently outperforms strong baselines including CTS-G and CTS-B, and we open-source our implementation for reproducibility.

\end{abstract}

\section{Introduction}\label{sec:intro}
We revisit \emph{combinatorial Thompson sampling (CTS)} for \emph{stochastic semi-bandits with adversarial sleeping arms} (hereafter \emph{sleeping semi-bandits}). 
In this setting, an agent repeatedly selects a \emph{super arm}, i.e., a feasible subset of at most $m$ arms, from a dynamically changing set of available arms drawn from $N$ base arms, subject to predefined combinatorial constraints.
Each base arm $i\in[N]$ has an unknown but fixed reward distribution with mean $\mu_i$.
In each round, the agent observes the individual rewards of the played arms (semi-bandit feedback) and receives the sum reward of the selected subset.
The goal is to maximize the cumulative reward over $T$ rounds, equivalently to minimize regret with respect to an oracle that selects the best feasible subset among the available arms at each round.

Sleeping semi-bandits provide a natural abstraction for networking systems with time-varying resource availability and partial feedback
\cite{kleinberg2010regret,Chatterjee2017AnalysisOT,hu2019int,NIPS2016_6450,NIPS2014_5381}.
For example, in routing, link availability may fluctuate due to congestion, failures, or maintenance, and the algorithm must select feasible paths under the current topology while only observing the performance of chosen links (e.g., delay or throughput). 
In wireless scheduling, user or channel availability can change over time due to fading or interference, and the scheduler selects a subset of users under interference/hardware constraints and observes feedback only for scheduled users.
Moreover, sleeping semi-bandits serve as a core building block for more involved networking models, including variants with fairness constraints
\cite{li2019combinatorial,li2019combinatorial1,wu2024achieving,wu2025low}.

A central challenge in sleeping semi-bandits is that the identity of the optimal action is \emph{time-varying}: the best feasible subset depends on the currently available arms, which may change adversarially over time.
This coupling between \emph{learning} (unknown reward means) and \emph{non-stationary feasibility} (sleeping constraints) makes regret minimization substantially more delicate than in standard semi-bandits, where the optimal action is fixed.
Consequently, an algorithm must efficiently balance exploration of unknown rewards with exploitation of the best \emph{currently available} options.

Two canonical approaches for balancing exploration and exploitation are \emph{upper confidence bounds (UCB)} and \emph{Thompson sampling (TS)}.
Both are inspired by \emph{optimism in the face of uncertainty} and admit efficient combinatorial implementations in semi-bandit problems.
For sleeping semi-bandits, UCB-style algorithms enjoy well-developed theory, including both problem-dependent and worst-case regret guarantees.
In contrast, despite the strong empirical performance of TS in many bandit applications~\cite{chapelle2011empirical}, its \emph{worst-case} theory for sleeping semi-bandits remains poorly understood. 
Such worst-case guarantees are particularly important here: instance-dependent quantities such as reward gaps can be unknown, unstable, or ill-defined under a time-varying action space, and worst-case regret provides a robust benchmark for algorithm design and comparison.

\textbf{Long-standing limitations of CTS for sleeping semi-bandits.}
CTS extends TS to combinatorial action sets by sampling per-arm estimates and selecting the optimal feasible subset under these samples.
While CTS is conceptually simple and widely used, existing analyses leave three major gaps:

\begin{enumerate}
    \item \textbf{No meaningful worst-case regret bounds.}
    Existing CTS results in semi-bandits are primarily \emph{problem-dependent}, typically scaling as $O(\log T/\Delta)$ for reward gap $\Delta$
    \cite{wang2018thompson,perrault2020statistical,zhang2021suboptimality,zhangthompson}.
    However, converting such results into worst-case bounds has remained elusive. 
    For \emph{CTS with Beta priors (CTS-B)}, the dependence on $m$ and $N$ can be exponential and is provably unavoidable~\cite{zhang2021suboptimality}.
    For \emph{CTS with Gaussian priors (CTS-G)}, the best known bounds still exhibit unfavorable polynomial dependence on $m$ and $N$
    \cite{zhangthompson}.
    As a result, a \emph{frequentist worst-case} understanding of CTS in semi-bandits (even \emph{without} sleeping arms) has remained incomplete.

    \item \textbf{No guarantees under adversarial arm availability.}
    Prior work has studied TS with \emph{stochastic} arm availability
    \cite{Chatterjee2017AnalysisOT}, but regret guarantees are missing when availability evolves adversarially.
    This adversarial model is well-motivated in networking systems, where availability can be shaped by unpredictable dynamics, failures, and external interference.

    \item \textbf{Subpar empirical performance of CTS-G.}
    Although CTS-G is a natural candidate for worst-case analysis due to its Gaussian structure, it often underperforms both UCB-based methods and CTS-B in practice, suggesting inefficient or uncoordinated exploration.
\end{enumerate}

\textbf{Our contributions.}
We resolve a long-standing gap in the theory of combinatorial Thompson sampling by providing the first \emph{worst-case} analysis of the \emph{standard} CTS algorithm with Gaussian priors (CTS-G), and then further refine CTS-G to obtain tighter guarantees and improved empirical performance.

\begin{itemize}
    \item \textbf{First worst-case regret analysis for standard CTS-G.}
    We present a new frequentist analysis of CTS-G and establish the first meaningful worst-case regret bound for combinatorial Thompson sampling in semi-bandits with (adversarially) sleeping arms. 
    In particular, we prove a worst-case regret upper bound of $\tilde{O}(m\sqrt{NT})$ and a matching lower bound of $\tilde{\Omega}(m\sqrt{NT})$ for CTS-G. This result directly bridges the long-standing theoretical gap where prior CTS analyses were predominantly gap-dependent and failed to provide reasonable worst-case guarantees

    Our proof strategy departs from classical gap-based analyses and instead directly controls the worst-case regime. 
    The key ingredients include a ghost-sample argument and a refined anti-concentration analysis for correlated Gaussian estimates.

    \item \textbf{Algorithmic refinement: a simple modification with tighter bounds.}
    Building on the above analysis, we propose \emph{Combinatorial Learning with a Single Gaussian seed (CL-SG)}, a lightweight variant of CTS-G that draws a \emph{single shared} Gaussian seed per round to coordinate exploration across arms.
    We show that CL-SG achieves a strictly improved worst-case regret bound
$\tilde{O}\!\left(\sqrt{mNT}\right)$,
    and we further prove a matching lower bound of $\tilde{\Omega}(\sqrt{mNT})$.


    \item \textbf{Improved empirical performance and reproducibility.}
    Experiments on real-world datasets demonstrate that CL-SG consistently outperforms CTS-G and competitive baselines such as CTS-B.
    We release our implementation and experimental pipeline as open source to support reproducibility and future research.
\end{itemize}

\section{The Sleeping Semi-Bandit Problem}
\label{sec:problem}


We consider a sleeping semi-bandit problem with a fixed set of $N$ base arms denoted by $[N]:= \{1, 2, \ldots, N\}$. Each base arm $a \in [N]$ is associated with a fixed but unknown reward distribution $p_a$ supported on $[0,1]$ with its mean denoted by $r_a$. 
Denote by $\Theta \subseteq 2^{[N]}$ the \emph{feasible set} consisting of all possible \emph{solutions} satisfying some certain constraints.  Each feasible solution $A \in \Theta $  can be viewed as  a super arm, which can be made up of more than one base arm. Let $m := \mathop{\max}_{A \in \Theta} |A|$ denote the maximum cardinality among all super arms, i.e., the maximum number of base arms in a super arm.

Different from the standard combinatorial bandits setting \cite{kveton2015tight}, where the learning agent faces up to a fixed decision set $\Theta$ in all the rounds,
in the sleeping semi-bandit setting, in each round $t = 1, 2, \ldots, T$, a time-varying feasible set $\Theta_t \subseteq \Theta$ is revealed to the learning agent. The feasible set could be generated in an adversarial way.
Then, the learning agent plays a super arm $A_t \in \Theta_t $, observes the random rewards $X_{a,t} \sim p_a$ for all the base arms $a \in A_t$, and obtains a reward $\sum_{a \in A_t} X_{a,t}$.  The goal of the learning agent is to choose a sequence of super arms to play to accumulate as much reward as possible over a finite number of $T$ rounds. Since $\Theta_t$ is revealed at the beginning of each round $t$, we let $A_t^* := \mathop{\arg \max}_{A \in \Theta_t} \sum_{a \in A} r_a$ denote the optimal super arm in round $t$. Then, the $T$-round (pseudo)-regret can be expressed as
\begin{equation}
         \mathcal{R}(T) := \sum_{t=1}^T \mathbf{E} \left[  \sum_{a \in A_t^*} r_{a} -   \sum_{a \in A_t} r_{a} \right],
    \label{def of regret}
\end{equation}
where the expectation is taken over $\Theta_t$ and $A_t$. 
Note that $A_t^*$ is also random, which is determined by $\Theta_t$.

\section{Related Works}\label{sec:RelatedWorks}
Given the foundational importance and practical relevance of UCB and TS in stochastic bandits, our discussion will primarily focus on adapting these algorithms for stochastic semi-bandits and stochastic sleeping semi-bandits.

\textbf{Semi-Bandits.}
Semi-bandits are a special case of combinatorial bandits~\cite{gai2012combinatorial,combes2015combinatorial}, where the reward of each played base arm can be observed. The performance of UCB-based algorithms for semi-bandits has been well studied.  A sublinear problem-dependent regret upper bound is derived in \cite{chen2013combinatorial} for a UCB-based algorithm called CombUCB. Later, the authors of~\cite{kveton2015tight} not only improved the problem-dependent regret bound to $O(mN \ln (T)  / \Delta)$ but also derived an $O(\sqrt{mNT\ln T})$  worst-case regret bound for CombUCB. In \cite{kveton2015tight} and~\cite{merlis2020tight}, an $\Omega(\sqrt{mNT})$ minimax regret lower bound was derived for the combinatorial bandits. When the reward distributions are mutually independent, it is proved in \cite{NIPS2016_e816c635} that the UCB-based algorithm can achieve a better problem-dependent regret bound of $O( N \ln^2 (m) \ln (T)  / \Delta)$, and a worst-case regret bound of $\sqrt{N\ln^2(m)T\ln T}$.

Regarding TS-based algorithms for semi-bandits, the authors of \cite{wang2018thompson} proved the first $O\left(m N \ln (T) / \Delta\right)$ problem-dependent regret bounds of CTS-B. The idea of CTS-B is to use Beta distributions to model the mean reward of each arm's reward distribution. Then, the authors of \cite{perrault2020statistical} improved the results of \cite{wang2018thompson} to $O\left(\frac{ N(\ln m)^2}{\Delta} \ln T+\frac{N m^3}{\Delta^2}+m\left( \frac{m^2+1}{\Delta}\right)^{2+4 m}\right)$, when the arm distributions are mutually independent.
However, both the problem-dependent bounds in \cite{wang2018thompson} and \cite{perrault2020statistical} contain a term that exponentially increases with the size of the optimal solutions. Later, the authors of \cite{zhang2021suboptimality} proved that this exponential term is unavoidable for CTS-B. Subsequently, using Gaussian priors, a significant improvement is made in \cite{zhangthompson} by reducing this exponential dependency to a polynomial term: $O\left(\frac{ N \ln m}{\Delta} \ln T+\frac{ N^2 m \ln m}{\Delta} \ln \ln T+P\left(m, N, \frac{1}{\Delta}, \Delta \right)\right)$. However, the polynomial term $P\left(m, N, \frac{1}{\Delta}, \Delta\right)$ still has a degree approximately $30$ in $m$ and $10$ in $N$. Thus, it is still difficult to obtain reasonable worst-case regret bounds from tuning the problem-dependent bounds~\cite{lattimore2020bandit}.

On the other hand, the authors of~\cite{huyuk2020thompson} gave an $O\left(\max \left\{N \sqrt{T \ln T}, N^2\right\}\right)$  worst-case Bayesian regret bound of CTS-B.

Thus, while problem-dependent bounds of TS-based algorithms are well-studied, non-Bayesian worst-case bounds of TS-based algorithms for semi-bandits have remained an open challenge for a long time.

\textbf{Sleeping Semi-Bandits.}
All the aforementioned works assume that the arm set from which the learning agent can play is fixed over all $T$ rounds, i.e., all the arms are always available and ready to be played. However, in practice, some of the arms may not be available in some rounds. Therefore, a bunch of literature studied the setting of sleeping semi-bandits~\cite{kleinberg2010regret,Chatterjee2017AnalysisOT,hu2019int,NIPS2016_6450,NIPS2014_5381,li2019combinatorial,li2019combinatorial1}.  In sleeping bandits, the set of available arms for each round, i.e., the availability set, can vary. For the simplest version of sleeping bandits, the problem-dependent regret bounds of UCB-based algorithms and TS-based algorithms have been analyzed in \cite{kleinberg2010regret} and \cite{Chatterjee2017AnalysisOT}, respectively. Regarding the sleeping semi-bandits, the authors of~\cite{hu2019int} proposed a UCB-based algorithm and derived a problem-dependent regret bound. The authors of~\cite{li2019combinatorial1} studied a variant of sleeping semi-bandits with fairness constraints, and if relaxing the fairness constraints, they gave a worst-case bound of $O(\sqrt{mNT\ln T})$ for UCB-based algorithms. Both the above works assume a stochastic availability set. The same-order worst-case upper regret bound for UCB on a non-stochastic (adversarial) availability set was also obtained in~\cite{abhishek2021sleeping}.

However, no prior work has established a \textbf{worst-case}, \textbf{frequentist regret bound} for CTS in semi-bandits with stochastic or \textbf{adversarial arm availability}. Since CTS-G has demonstrated more favorable scaling in existing problem-dependent bounds~\cite{zhangthompson}, raising the question of whether a reasonable worst-case guarantee is achievable. This motivates our study of CL-SG, where the regret analysis of CL-SG can be easily extended to CTS-G and prove that a reasonable worst-case regret bound does, in fact, exist.

\section{Gaussian Randomized Algorithms}
\label{sec:algorithm}
In this section, we bridge the aforementioned gaps by first presenting CTS-G, an algorithm enjoying $\tilde{O}(m\sqrt{NT})$ and $\tilde{\Omega}(m\sqrt{NT})$ regret upper and lower bounds, respectively.  Then, we propose CTS-G, an algorithm enjoying $\tilde{O}(m\sqrt{NT})$ and $\Omega(\sqrt{mNT \ln \frac{N}{m}})$ regret upper and lower  bounds, respectively. All the detailed proofs can be found in Appendix~\ref{pfCTS2} and~\ref{pfs4CL-G}.

Before describing the algorithms, we first introduce some notations specific to this section.
Let $n_{a,t}:=\sum_{\tau = 1}^{t-1} \mathbf{1}[a \in A_\tau]$ denote the total number of times that base arm $a\in [N]$ has been pulled at the beginning of round $t$.
Let $ \hat{r}_{a, n_{a,t}}:= \frac{\sum_{\tau = 1}^{t-1} \mathbf{1}[a \in A_\tau] \cdot r_{a,\tau} }{n_{a,t}}$
denote the empirical mean of base arm $a$ at the beginning of round $t$, which is the average of $n_{a,t}$ i.i.d. random variables according to  reward distribution $p_a$. Let $\mathcal{F}_t := \left\{r_{a,\tau}, \forall a \in A_{\tau}, \forall \tau \in [t] \right\}$ collect all the history information up to the end of round $t$.


\subsection{Combinatorial Thompson Sampling with Gaussian Priors~(CTS-G)}\label{sec: ALG1}

\begin{algorithm}[h]
\caption{Combinatorial Thompson Sampling with Gaussian Priors~(CTS-G)}\label{alg:multiseed}
\begin{algorithmic}[1]
\Require {arm set $[N]$, exploration rate $\gamma$}
\State Initialize $n_{a,1} = 0$ and $\hat{r}_{a,n_{a,1}} = 0$ for all base arms $a \in [N]$
\For{$t = 1, 2, \ldots $}
\State Observe feasible set $\Theta_t$
\State Draw $w_{a,t} \sim \mathcal{N}(\hat{r}_{a, n_{a,t}}, \frac{ \gamma m \ln t}{n_{a,t} + 1})$ for each base arm $a \in [N]$
\State Play  super arm $A_t =  \mathop{\arg \max}\limits_{A \in \Theta_t} \sum_{a \in A} w_{a,t}$
\State Observe  $r_{a,t} \sim p_{a}$ for all base arms $a \in A_t$  and update $n_{a,t}$ and $\hat{r}_{a, n_{a,t}}$ for all $a \in A_t$.
\EndFor
\end{algorithmic}
\end{algorithm}


CTS-G presented in Alg.~\ref{alg:multiseed} is a direct adaptation of TS with Gaussian priors \cite{Agrawal2017} to the sleeping semi-bandit problems. The core idea is to use posterior distributions to model the mean reward $r_a$ of each base arm $a \in [N]$. 
In each round $t$, CTS-G  draws a Gaussian posterior sample $w_{a,t} \sim \mathcal{N}(\hat{r}_{a, n_{a,t}},\frac{{\gamma m\ln t}}{n_{a,t}+1})$ for each $a \in [N]$, where $\gamma > 0$ is a constant to control the exploration level.\footnote{In practice, we only need to draw posterior samples for available arms to improve efficiency.} We can view the collection $\bm{w}_t = \left\{ w_{a,t}, \forall a \in [N]\right\} $ of all posterior samples as the ``sampled problem instance'' based on which the learning agent conducts learning in round $t$.
Then, based on the revealed feasible set $\Theta_t$, CTS-G plays the super arm $A_t \in \mathop{\arg \max}_{A \in \Theta_t} \sum_{a \in A}w_{a,t}$ with the highest aggregated value of posterior samples and observes each individual base arm's random reward.

\begin{theorem}\label{thm:upperCTS-G}
    (1) The regret of CTS-G is ${O} \left(m\ln(T)\sqrt{NT } \right)$. (2) There exists a semi-bandit problem instance such that CTS-B suffers at least regret of $\Omega( m\sqrt{NT \ln T} )$.
\end{theorem}

\paragraph{Discussion.}
Theorem~\ref{thm:upperCTS-G} states that CTS-G is worst-case optimal up to a logarithmic factor. 
Compared with UCB-based algorithms for sleeping semi-bandits, our upper bound has an extra factor of $\sqrt{m \ln T}$ with the ones by \cite{hu2019int} and \cite{li2019combinatorial1}, which are $O(\sqrt{mNT \ln T})$. However, it is important to note a significant aspect of our model: unlike the assumptions in \cite{hu2019int} and \cite{li2019combinatorial1}, our bound is derived without relying on stochastic assumptions regarding the availability of arms. 
Furthermore, the upper bound is minimax optimal up to an extra ${\ln (T)}\sqrt{m}$ factor as compared to the $\Omega \left(\sqrt{mNT} \right)$ minimax lower bound for combinatorial bandits shown in \cite{merlis2020tight}. 

\paragraph{Upper bound proof sketch.}
The theoretical analysis is non-trivial due to overlapping base arms among super arms. Additionally,  the optimal super arm $A_t^*$ is dynamic and unobservable, as only the played super arm $A_t$ is visible in each round $t$.
To decompose the regret, we define a high-probability event for the empirical estimates. Let $\mathcal{E}_t:= \left\{|r_a - \hat{r}_{a, n_{a,t}}|\leq \sqrt{\frac{3\ln (Nt)}{n_{a,t}+1}}, \forall a \in [N]\right\}$ be the event that the empirical means are close to their true means by the beginning of round $t$.
Let $t^{\prime}=\max \{\sqrt{m}, 4\}$ and $\mathbf{E}_{\Theta_t}[\cdot]:=\mathbf{E}[\cdot \mid \Theta_t]$.\footnote{
We note that such a definition $\mathbf{E}_{\Theta_t}[\cdot]$ applies pointwise for any realized $\Theta_t$, with no distributional assumption made on the availability process.} Then, we decompose the regret defined in (\ref{def of regret}) as
\begin{equation}
\begin{aligned}
        \mathcal{R}(T) &\leq \underbrace{\sum_{t=t^\prime}^T  \mathbf{E} \left[  \sum_{a \in A_t^*} r_{a} - \mathbf{E}_{\Theta_t}  \left[   \sum_{a \in A_t}w_{a,t}    \right] \right] }_{=:I_1,~\text{optimism term}} \\
        & +  \underbrace{\sum_{t=t^\prime}^T  \mathbf{E} \left[ \mathbf{E}_{\Theta_t} \left[ \sum_{a \in A_t}  \left( w_{a,t}  -  r_{a}  \right) \mathbf{1}[\mathcal{E}_t]\right] \right]}_{=:I_2,~\text{deviation term}} \\
        & +  mt^\prime + O(1).    
\end{aligned}
\end{equation}

The deviation term $I_2$ is easy to analyze as we can observe $A_t$, and is upper bounded by $\tilde{O}(m \sqrt{NT})$ via using concentration bounds.
The central question is how to upper bound the optimism term, which measures the gap between the \emph{maximum amount of true reward $\sum_{a \in A_t^*} r_{a}$} that the learning agent could achieve and the \emph{expected maximum amount of reward $\sum_{a \in A_t}w_{a,t} $} that the learning agent can observe in round $t$. 
Intuitively, if the learning agent is lucky, i.e., the history $\mathcal{F}_{t-1}$ gives $ \sum_{a \in A_t^*} r_{a} \le \mathbf{E} _{\Theta_t}\left[ \sum_{a \in A_t}w_{a,t} \right] $, there is no regret in round $t$ for this term.   
Let $(\cdot)^{+} := \max \left\{ \cdot, 0 \right\}$ be an activation function. Then, we have
\begin{equation}
\resizebox{1\hsize}{!}{$
    \begin{aligned}
          &\sum_{a \in A_t^*} r_{a} - \mathbf{E}_{\Theta_t}  \left[   \sum_{a \in A_t}w_{a,t}    \right]  \le \left(  \sum_{a \in A_t^*} r_{a} - \mathbf{E}_{\Theta_t}  \left[   \sum_{a \in A_t}w_{a,t}    \right] \right)^+ .        
    \end{aligned}
    $}
\end{equation}

Let $c(\gamma)$ be a constant only depending on $\gamma$.  
In our novel technical Lemma~\ref{lm:1}, inspired by \cite{russo2019worst}, we show 
\begin{equation}
    \begin{aligned}
          &\left(  \sum_{a \in A_t^*} r_{a} - \mathbf{E}_{\Theta_t}  \left[   \sum_{a \in A_t}w_{a,t}    \right] \right)^+ \\ 
          & \leq   c(\gamma) \cdot \mathbf{E}_{\Theta_t}\left[  \left(  \sum\limits_{a\in A_t} {w}_{a,t} - \mathbf{E}_{\Theta_t}\left[\sum\limits_{a\in A_t} {w}_{a,t} \right]  \right)^+  \right],
    \end{aligned}
\end{equation}
which tackles the challenge brought by the unobservability of $A_t^*$.

Next, via introducing an independent ``ghost'' copy $\tilde{w}_{a,t} \sim  \mathcal{N}(\hat{r}_{a, n_{a,t}}, \frac{ \gamma m \ln t}{n_{a,t} + 1})$ of $w_{a,t}$, we show 
\begin{equation}
    \begin{aligned}
        &\mathbf{E}_{\Theta_t}\left[  \left(  \sum\limits_{a\in A_t} {w}_{a,t} - \mathbf{E}_{\Theta_t}\left[\sum\limits_{a\in A_t} {w}_{a,t} \right]  \right)^+   \right]  \\
        & \le 
          \mathbf{E}_{\Theta_t} \left[ \left| \sum\limits_{a\in A_t} \left(w_{a,t} -  \tilde{w}_{a,t} \right) \right|  \right],        
    \end{aligned}
\end{equation}
which gets rid of the introduced activation function.

Since $w_{a,t} -  \tilde{w}_{a,t} \sim \mathcal{N}\left(0, \frac{2\gamma m \ln t}{n_{a,t}+1}\right)$, we only need to deal with Gaussian random variables and have
\begin{equation}
    \begin{array}{l}
         \sum\limits_{t= t'}^{T} \mathbf{E}\left[ \left| \sum\limits_{a\in A_t} \left( w_{a,t} -  \tilde{w}_{a,t} \right) \right|  \right]   \leq O \left(m\ln T \sqrt{\gamma NT} \right).
    \end{array}
\end{equation}

\paragraph{Lower bound proof sketch.}
To establish the $\Omega(m\sqrt{NT \ln T})$ lower bound for CTS-G, we construct a top-$m$ problem with $N$ base arms where any $m$ arms can be selected per round. We set $N \geq 400m$ and consider a deterministic reward setting where the optimal $m$ arms (set $G$) have reward $\Delta = \frac{4}{5}\sqrt{\frac{N\ln T}{T}}$, and all other arms have zero reward. The total regret $\mathcal{R}(T)$ is proportional to the expected number of suboptimal arms $k_t$ played, i.e., $\mathcal{R}(T) = \Delta \mathbf{E}[\sum_{t=1}^T k_t]$. Note that the total number of pulls accounting for all arms by the end of round $t-1$ is exactly $m(t-1)$. We define $K_{t-1}$ as the cumulative count of suboptimal arm pulls up to round $t-1$ and analyze two mutually exclusive and exhaustive cases:
\begin{itemize}
    \item Case 1: excessive pulls from sub-optimal arms. Let $H = O\left(m\ln(T)/\Delta^2 \right)$. If there exists a round $t_0 \in (T/2, T]$  such that the total number of pulls from sub-optimal arms $K_{t_0-1} > \frac{1}{2} m (t_0-H)$, it implies the agent has already pulled suboptimal arms too many times. In this case, the regret is immediately lower bounded by $\Delta K_{t_0-1} = \Omega(m\sqrt{NT \ln T})$.
    \item  Case 2: sufficient  pulls from optimal arms and insufficient pulls from sub-optimal arms. If $K_{t-1} \leq \frac{1}{2} m (t-H)$ for all later rounds $t \ge T/2$, we prove that there remains a constant probability $p_0^2$ that at least $0.1m$ suboptimal arms are chosen in each round, where $p_0 >0$ is a constant. This is achieved by showing two events that happen with a constant probability:
    \begin{enumerate}
        \item Sufficient pulls from optimal rams. Let $Y^*_t$ denote the event that at least $0.1m$ optimal arms have posterior samples $w_{a,t}$ below a threshold $\tau = \Delta + \frac{\Delta}{4}$. We show that for any $t \ge \frac{T}{2}$, $\Pr(Y_t^* \mid \mathcal{F}_{t-1})\geq  p_0$. 
        \item Insufficient pulls from sub-optimal arms.  Let $Y_t$ denote the event that at least $0.1m$ suboptimal arms have posterior samples $w_{a,t}$ exceeding $\tau$ due to insufficient observation (over-exploration). We show that for any $t \ge \frac{T}{2}$, $\Pr(Y \mid \mathcal{F}_{t-1})\geq  p_0$.
    \end{enumerate}
\end{itemize}
Therefore, we have
    \begin{equation}
    \begin{aligned}
\mathbf{E}\left[\mathcal{R}_T\right] & =\Delta \sum_{t=1}^T \mathbf{E}\left[k_t\right] = \Delta \sum_{t=1}^T \operatorname{Pr}\left(Y_t^*, Y_t\right) \cdot \mathbf{E}\left[k_t \mid Y_t^*, Y_t\right]\\
& \geq \Delta \sum_{t= \lceil\frac{T}{2}\rceil}^T \mathbf{E}\left[\operatorname{Pr}\left(Y_t^* \mid \mathcal{F}_{t-1}\right) \cdot \operatorname{Pr}\left(Y_t \mid \mathcal{F}_{t-1}\right)\right] \cdot 0.1 m\\
&\geq \Delta \cdot T/2 \cdot p_0^2 \cdot m = \Omega(m \sqrt{N T \ln T}).
\end{aligned}
    \end{equation}

Note that our proof is nontrivial in that the analysis of Case 2 combines a conservation-of-counts/pigeonhole step and a Chernoff “bulk deviation” argument, yielding that $\Omega(m)$ suboptimal arms are selected in many late rounds with constant probability. 

\subsection{Combinatorial Learning with a Single Gaussian Seed~(CL-SG)}\label{sec: ALG2}
\begin{algorithm}[h]
\caption{Combinatorial Learning with a Single Gaussian Seed~(CL-SG)}\label{alg:singleseed}
\begin{algorithmic}[1]
\Require {arm set $[N]$, exploration rate $\gamma$}
\State Initialize $n_{a,1} = 0$ and $\hat{r}_{a,0} = 0$ for all base arms $a \in [N]$
\For{$t = 1, 2, \ldots$}
\State Observe feasible set $\Theta_t$
\State Draw $w_t \sim \mathcal{N}(0, 1)$
\State Construct $\bar{r}_{a,t} = \hat{r}_{a, n_{a,t}} + w_t \cdot \sqrt{\frac{{\gamma \ln t}}{n_{a,t} + 1}}$ for all base arms $a \in [N]$
\State Play super arm $A_t = \mathop{\arg \max}\limits_{A \in \Theta_t} \sum\limits_{a \in A} \bar{r}_{a,t}$
\State Observe  $r_{a,t} \sim p_{a}$ for all base arms $a \in A_t$  and update $n_{a,t}$ and $\hat{r}_{a, n_{a,t}}$ for all $a \in A_t$.
\EndFor
\end{algorithmic}
\end{algorithm}

The CTS-G algorithm still has a gap of $\sqrt{m}$ from the minimax-optimal lower bound of $\Omega(\sqrt{mNT \ln T})$, because we ``pay'' additional $m$ in the variance of per-arm posterior distribution. We argue that this $m$ is difficult to be removed, because controlling $\Pr(\sum_{a\in A^*} w_{a,t}\ge \sum_{a\in A^*} r_a)$ requires a constant anti-concentration lower bound for a sum of $m$ correlated estimation errors. This in turn forces the injected Gaussian noise to scale with $m$. To avoid this intrinsic set-level noise inflation, we introduce CL-SG, eliminating the extra $m$ factor in the variance.

Inspired by \cite{xiong2021near}, we devise CL-SG which enjoys a $\tilde{O}(\sqrt{mNT})$ regret bound. The key idea behind the removal of the extra $\sqrt{m}$ factor as compared to the regret of CTS-G (Alg.~\ref{alg:multiseed}) is that CL-SG uses a single random seed $w_t \sim \mathcal{N}(0,1)$  to perturb the empirical estimates of all the base arms, as shown in Alg.~\ref{alg:singleseed}.  
After drawing $w_t$, we construct $\bar{r}_{a,t} = \hat{r}_{a, n_{a,t}} + w_t \cdot \sqrt{\frac{\gamma \ln t}{n_{a,t}+1}}$ for all the base arms $a \in [N]$, where constant $\gamma > 0$ controls the exploration level. Then, we play  $A_t = \mathop{\arg \max}_{A \in \Theta_t} \sum_{a \in A} \bar{r}_{a,t}$ from the feasible set $\Theta_t$ in round $t$.

\begin{theorem}\label{thm:upperCL-SG}
        (1) The regret of CL-SG is ${O} \left(\ln (T)\sqrt{mNT} \right)$. (2) There exists a problem instance such that CL-SG suffers  $\Omega \left(\sqrt{mNT \ln T} \right)$ regret.
\end{theorem}
\paragraph{Discussion.} Theorem~\ref{thm:upperCL-SG} states that CL-SG improves the upper bound of CTS-G by a factor of $\sqrt{m}$. To the best of our knowledge, the above bounds are currently the best problem-independent results for TS-based algorithms in sleeping semi-bandits with either the stochastic or adversarial availability of arms.

\paragraph{Upper bound proof sketch.}
The extra $\sqrt{m}$ in CTS-G comes from the $m$ factor in the variance of the Gaussian posterior sample $w_{a,t}$, necessary to keep $c(\gamma)$ bounded by a constant. To bound $c(\gamma)$, we must lower bound $\Pr_{\Theta_t}\left(\sum_{a \in A_t^*} w_{a, t}-\hat{r}_{a, n_{a, t}} \geq \sum_{a \in A_t^*} \sqrt{\frac{4 \ln t}{n_{a, t}+1}}\right)$, requiring the Cauchy-Schwarz inequality to bring the summation inside the square root for the RHS term in the probability, which scales with $\sqrt{m}$, i.e., $\sum_{a \in A_t^*} \sqrt{\frac{4 \ln t}{n_{a, t}+1}} \leq \sqrt{m \sum_{a \in A_t^*} \frac{4 \ln t}{n_{a, t}+1}}$. This fact further results in an extra $m$ in the variance of CTS-G Gaussian seeds for the probability to be lower bounded by a constant. On the other hand, with CL-SG, using a single $w_t$, we lower bound a similar probability, $\Pr\left(\sum_{a\in A_t^*} w_t \sqrt{\frac{\gamma \ln t}{n_{a,t}+1}}  \geq \sum_{a\in A_t^*} \sqrt{\frac{4\ln t}{n_{a,t}+1}}\right)$, by directly dividing both sides by $\sum_{a \in A_t^*} \sqrt{\frac{4 \ln t}{n_{a, t}+1}}$, which avoid the extra $m$ in the variance.

\paragraph{Lower bound proof sketch.}
The fundamental challenge in analyzing CL-SG arises from the shared Gaussian random seed $w_t \sim \mathcal{N}(0, 1)$ used across all base arms to drive exploration. Unlike standard CTS-G, which samples independently per arm, this mechanism induces complex dependencies among all available super arms. 

To establish the $\Omega(\sqrt{mNT \ln T})$ lower bound for CL-SG, we consider a top-$m$ problem with $N = 2m$ base arms. We set the optimal $m$ arms (set $G$) to have a deterministic reward $\Delta = \sqrt{\frac{N \ln T}{10^4 m T}}$, and all other arms to have zero reward.
The total regret $\mathcal{R}(T)$ is proportional to the expected number of suboptimal arms $k_t$ played, i.e., $\mathcal{R}(T) = \Delta \mathbf{E}[\sum_{t=1}^T k_t]$.
We define $G^*_t$ as the set of optimal arms pulled more than $t-1-c_1 T$ times (where $c_1 = 10^{-6}$), and define $G_t$ as sub-optimal base arms that have not been observed sufficiently by the end of round $t-1$.

Then, we analyze two mutually exclusive and exhaustive cases based on the event  $B_t^* := \{ |G^*_t| \ge 0.9995m \}$.

\begin{itemize}
\item Case 1: Excessive Exploration. If there exists a round $t$ where $B_t^*$ is false, the total number of optimal arm pulls is significantly lower than $m(t-1)$. Consequently, the cumulative suboptimal arm pulls $\sum_{s=1}^{t_0-1} k_s$ must exceed $0.0005mc_1T$, directly yielding a regret of $\Omega(\sqrt{mNT \ln T})$.
\item Case 2: Coordinated Selection via Shared Seed. If $B_t^*$ is true for all $t$, we focus on late rounds $t \geq \alpha T + 1$ (where $\alpha = \frac{255}{256} + c_1$). We show that when the shared seed $w_t$ falls within a specific constant interval $[l_1, l_2]$ with a constant probability $p_0$, the algorithm consistently prefers suboptimal arms. This is achieved by showing:
\begin{enumerate}
\item For $w_t \in [l_1, l_2]$, any sufficiently observed optimal arm $a \in G_t^*$ has a posterior sample $\bar{r}_{a,t} < \Delta + 100l_2\Delta$.
\item For the same $w_t$, any insufficiently observed suboptimal arm $b \in G_t$ (where $|G_t| > 0.92m$) has $\bar{r}_{b,t} \geq 100\sqrt{6}l_1\Delta$.
\end{enumerate}
\end{itemize}
By tuning $l_1 = \frac{1}{2\sqrt{2}}$ and $l_2 = 0.5$, we ensure $\bar{r}_{a,t} < \bar{r}_{b,t}$, forcing the agent to play $k_t \geq 0.92m$ suboptimal arms whenever $w_t \in [l_1, l_2]$. Therefore, we have
\begin{equation}
\begin{aligned}
\mathbf{E}\left[\mathcal{R}_T\right] 
& = \Delta \sum_{t=1}^T \mathbf{E}\left[k_t\right] \\
&\geq \Delta \sum_{t=\lceil \alpha T + 1 \rceil}^T \Pr(w_t \in [l_1, l_2]) \cdot \mathbf{E}[k_t \mid w_t \in [l_1, l_2]] \\
& \geq \Delta \cdot (1-\alpha)T \cdot p_0 \cdot 0.92m = \Omega(\sqrt{mNT \ln T}).
\end{aligned}
\end{equation}

\section{Experiments}
\label{sec:results}
In this section, we evaluate our algorithm in the context of network routing with sleeping semi-bandits. Specifically, we assess the performance of CL-SG against the standard CTS-G algorithm, which independently samples for each arm according to $\mathcal{N}\left(\hat{r}_{a, n_{a, t}}, \frac{\gamma m \ln t}{n_{a, t}+1}\right)$,  and we examine the effect of varying the exploration parameter $\gamma$.  In addition, we compare against the following baselines, each selecting the action $A_t = \arg\max_{A \in \Theta_t} \sum_{a \in A} \theta_{a,t}$ with $\theta_{a,t}$ defined as follows:
\begin{itemize}
    \item CTS-B~\cite{wang2018thompson}: $\theta_{a,t} \sim \text{Beta}(\hat{r}_{a,n_{a,t}} n_{a,t} + 1, n_{a,t} - \hat{r}_{a,n_{a,t}} n_{a,t} + 1)$.
    \item BG-CTS~\cite{zhangthompson}: $\theta_{a,t} \sim N\left(\hat{r}_{a,n_{a,t}}, 2g(t)\sigma^2 / n_{a,t} \right)$, with $\sigma^2 = 1/4$ for Bernoulli rewards and $g(t)$ defined as in~\cite{zhangthompson}.
    \item CombUCB~\cite{kveton2015tight}: $\theta_{a,t} = \hat{r}_{a,n_{a,t}} + \sqrt{1.5 \ln t / n_{a,t}}$.
\end{itemize}
We evaluate these algorithms under two routing scenarios:
\begin{itemize}
    \item Setting 1 (Synthetic Network): This is a controlled numerical experiment based on a wireless mesh network of $16$ nodes in a $4\times4$ grid with a total of $24$ links. Among the links, one predefined path of four hops yields Bernoulli rewards with mean $0.9$ per link, while the remaining links yield Bernoulli rewards with mean $0.8$. All links have a uniform availability probability of $0.75$.
    \item Setting 2 (Real-World Network): This setting uses real-world traces from the UCSB MeshNet dataset~\cite{1632477,ucsb_meshnet}, which provides per-minute neighborhood tables. Each row records the \emph{expected transmission time (ETT)} between a node and its neighbors. We define the reward for each link as $1 - \text{normalized ETT}$. Since link availability varies over time, this naturally fits into the sleeping semi-bandit framework.
\end{itemize}
All reported results are averaged over $100$ independent runs. 

\begin{figure}[htbp]
\centering
\begin{tikzpicture}[scale=1.2, every node/.style={circle, draw, minimum size=8mm}]

\foreach \x in {0,1,2} {
    \foreach \y in {0,1,2} {
        \ifthenelse{\x=0 \AND \y=0}
            {\node[fill=blue!20] (n\x\y) at (\x,\y) {S};}
            {\ifthenelse{\x=2 \AND \y=2}
                {\node[fill=red!20] (n\x\y) at (\x,\y) {T};}
                {\node (n\x\y) at (\x,\y) {};}
            }
    }
}

\foreach \x in {0,1} {
    \foreach \y in {0,1,2} {
        \pgfmathtruncatemacro{\xnext}{\x + 1}
        \draw[-] (n\x\y) -- (n\xnext\y);
    }
}

\foreach \x in {0,1,2} {
    \foreach \y in {0,1} {
        \pgfmathtruncatemacro{\ynext}{\y + 1}
        \draw[-] (n\x\y) -- (n\x\ynext);
    }
}

\end{tikzpicture}
\caption{Setting 1: Wireless Mesh Network.}
\end{figure}
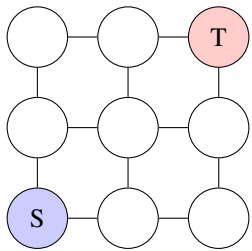

\paragraph{Comparison of the Regret}
\begin{figure}[h]
    \centering
    \begin{subfigure}[b]{0.35\textwidth}
        \centering
        \includegraphics[width=\textwidth]{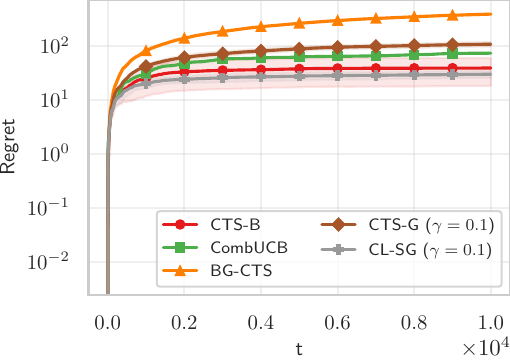}
        \caption{Setting $1$}
        \label{fig:regret_sett1}
    \end{subfigure}
    \begin{subfigure}[b]{0.35\textwidth}
        \centering
        \includegraphics[width=\textwidth]{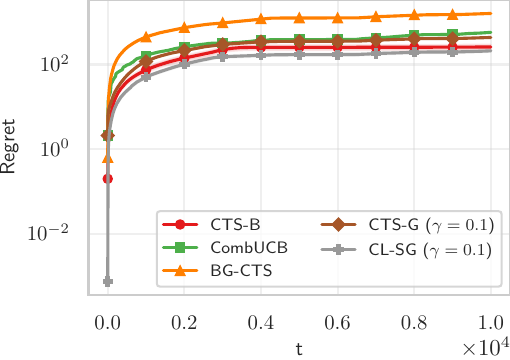}
        \caption{Setting $2$}
        \label{fig:regret_sett2}
    \end{subfigure}
    \caption{Comparison of the regret for both settings with $\gamma = 0.1$ for both CTS-G and CL-SG.}
    \label{fig:regret}
\end{figure}
The regret results over $T = 10^4$ rounds are shown in Fig.~\ref{fig:regret} with shaded areas indicating $97.5\%$ confidence intervals. The confidence intervals for some algorithms are not easily visible due to their small size.  

In both settings, CL-SG draws a minimal number of Gaussian random samples in each round, enhancing its efficiency. This reduction in randomness improves robustness, even for a large $\gamma$, preventing excessive exploration. As a result, CL-SG achieves superior performance, surpassing CTS-B, BG-CTS, and CTS-B.
This indicates the efficiency of CL-SG's design in optimizing the exploration-exploitation trade-off more effectively than its counterparts. 



\paragraph{Effect of Different Exploration Rates}\label{sec:difflambda}
We also compare the performance of CTS-G and CL-SG under different exploration rates, i.e., $\gamma = 0.01, 0.1, 0.5$, and $1$. The results are shown in Fig.~\ref{fig:explorationrate}.
\begin{figure}[h]
    \centering
    \begin{subfigure}[b]{0.35\textwidth}
        \centering
        \includegraphics[width=\textwidth]{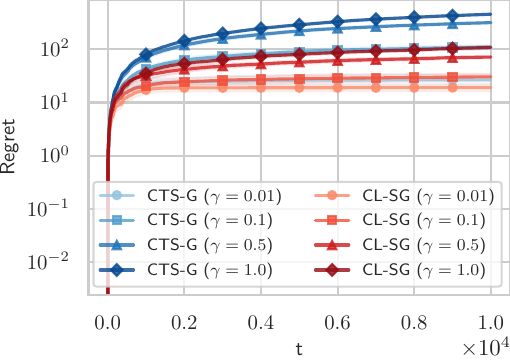}
        \caption{Setting $1$}
        \label{fig:gamma_sett1}
    \end{subfigure}
    \begin{subfigure}[b]{0.35\textwidth}
        \centering
        \includegraphics[width=\textwidth]{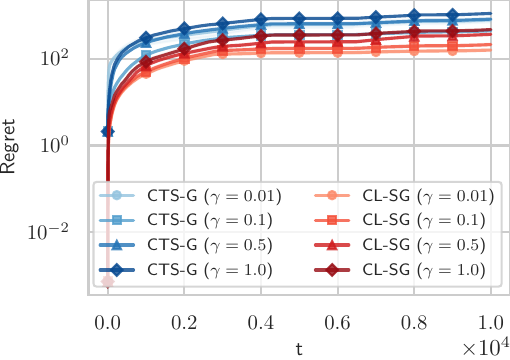}
        \caption{Setting $2$}
        \label{fig:gamma_sett2}
    \end{subfigure}
    \caption{Effect of different exploration rates $\gamma = 0.01, 0.1, 0.5$, and $1$ for CTS-G and CL-SG.}
    \label{fig:explorationrate}
\end{figure}
We can observe that in both settings, CL-SG and CTS-G achieve the lowest regret when $\gamma = 0.01$. If we continue to increase $\gamma$, both algorithms suffer a larger regret. This suggests that a certain lower level of randomness (exploration) is more effective in practice.

\section{Conclusion}
\label{sec:conclusion}
This paper addresses a long-standing open problem by establishing the first worst-case regret bounds for TS-based algorithms in sleeping semi-bandit settings. We first analyzed CTS-G, giving an upper bound of $\tilde{O}(m\sqrt{NT })$ and a matching lower bound of $\tilde{\Omega}(m\sqrt{NT})$. Next, we propose CL-SG, a variant of CTS-G that draws only a single shared Gaussian sample per round. CL-SG achieves near-optimal performance, with a worst-case upper bound of $\tilde{O}(\sqrt{mNT})$ and a matching lower bound of $\tilde{\Omega}(\sqrt{mNT})$.
Empirically, CL-SG significantly outperforms existing benchmarks such as CTS-B and CombUCB, demonstrating both improved accuracy and efficiency across diverse environments.

Looking ahead, we aim to theoretically characterize the optimal trade-off between randomness and exploration, and to further obtain the minimax-optimal regret bounds. The authors have provided public access to their code at \textbf{\url{https://tinyurl.com/ton26ts}}.

\bibliography{reference}
\bibliographystyle{IEEEtran}


\appendix

\section{Notations and Facts}
\textbf{Notations:} Let $\mathcal{F}_{t-1}$ denote by the history of past actions and rewards until the end of round $t-1$.  Recall that $\mathbf{E}_{\Theta_t}[\cdot]:=\mathbf{E}[\cdot \mid \Theta_t]$ and $\Pr_{\Theta_t} (\cdot) :=\Pr(\cdot \mid \Theta_t)$. Denote by $\mathcal{E}_t:= \left\{\forall{a \in [N]}:|r_a - \hat{r}_{a, n_{a,t}}|\leq \sqrt{\frac{3\ln Nt}{n_{a,t}+1}}\right\}$ the high-probability event that the empirical mean is close to the true mean reward for arm $a$, and by $\overline{\mathcal{E}_t}$ the complementary event of $\mathcal{E}_t$.
Recall that $\Tilde{w}_{a,t} \sim \mathcal{N}(\hat{r}_{a,n_{a,t}}, \frac{\gamma m \ln t}{n_{a,t}+1})$ is i.i.d. of $w_{a,t}$ for CTS-G, and $\Tilde{r}_t:= \hat{r}_{a, n_{a,t}} + \Tilde{w}_t\sqrt{\frac{\gamma \ln t}{n_{a,t}+1}}$, where $\Tilde{w}_t \sim \mathcal{N}(0,1)$ is i.i.d. of $w_t$ for CL-SG.

\begin{fact}\label{lm:gaussianineq}
    For a Gaussian distributed random variable $Z$ with mean $\mu$ and variance $\delta^2$, for any $z$, we have that
    \begin{equation}\label{eq:anti1}
        \frac{1}{4 \sqrt{\pi}} \cdot e^{-7 z^2 / 2} \leq \operatorname{Pr}(|Z-\mu|>z \sigma) \leq \frac{1}{2} e^{-z^2 / 2},
    \end{equation}
    and for any $z > 0$,
    \begin{equation}\label{eq:anti2}
        \Pr(Z - \mu>z \sigma) \geq \frac{1}{\sqrt{2 \pi}} \frac{z}{z^2+1} e^{-\frac{z^2}{2}}.
    \end{equation}
\end{fact}

\begin{fact}\label{lm:maximal4subgaussian}
    Let $X_1,\ldots,X_N$ be $N$ real random variables with $X_i \sim \operatorname{subG}\left(\sigma^2\right), i=1,\ldots, N$, not necessarily independent. Then,
    \begin{equation}
        \mathbb{E}\left[\max _{i=1, \ldots, N}\left|X_i\right|\right] \leq \sigma \sqrt{2 \log (2 N)}.
    \end{equation}
\end{fact}

\begin{fact}[Chernoff Bound]
Let $Y$ be a random variable, $\mu = \mathbf{E}[Y]$ and $\delta \in (0,1)$. We have 
    \begin{equation}\label{eq:chernoff}
        \operatorname{Pr}\left(Y \leq(1-\delta) \mu \right) \leq \exp \left(-\frac{\delta^2}{2} \mu\right).
    \end{equation}

\end{fact}

\begin{fact}[Hoeffding's Lemma]
Let $X$ be any real-valued random variable such that $a\leq X\leq b$ almost surely. Then, for all $\lambda \in \mathbb{R}$, we have
\begin{equation}\label{eq:hoeffding}
    \mathbf{E}\left[e^{\lambda(X-\mathbb{E}[X])}\right] \leq \exp \left(\frac{\lambda^2(b-a)^2}{8}\right).
\end{equation}
\end{fact}


\begin{fact}[Cantelli's Inequality] Let $X$ be a real-valued random variable with mean $\mu$ and variance $\sigma^2$. Then, for any $\lambda <0$, we have 
\begin{equation}\label{eq:cantelli}
    \operatorname{Pr}\left(X-\mu \ge \lambda\right) \ge 1 -\frac{\sigma^2}{\sigma^2 + \lambda^2}.
    \end{equation}
    \end{fact}


    

\section{Proofs for Theorem~\ref{thm:upperCTS-G}}\label{pfCTS2}
\subsection{Proof of Lemma~\ref{lm:1}}\label{sec:lm:1}
\setcounter{lemma}{0}
\begin{lemma}\label{lm:1}
    In any round $t \geq \max\{\sqrt{m}, 4\}$,  the optimism part in CTS-G satisfies that
    \begin{equation}
    \begin{aligned}
        &\mathbf{E}  \left[  \sum_{t=\max\{\sqrt{m},4\}}^T  \left( \sum_{a \in A_t^*} r_{a} - \sum_{a \in A_t}w_{a,t}  \right)  \right] \\
        &\leq 8\sqrt{3\gamma}\Phi(-\sqrt{4/\gamma})^{-1}m\ln T \sqrt{NT}.        
    \end{aligned}
    \end{equation}
\end{lemma} 

\begin{proof}
For each  $a \in [N]$, we let $\tilde{w}_{a, t} \sim \mathcal{N} \left(\hat{r}_{a, n_{a,t}}, \frac{m \gamma \ln t}{n_{a, t}+1} \right)$ be an independent copy of $w_{a,t}$.  Let $(\cdot)^{+} := \max \left\{ \cdot, 0 \right\}$. Let $\bf{w}$ collect all the Gaussian random variables $w_{a,t}$ for all $a \in [N]$. 
Recall that $\mathbf{E}_{\Theta_t}[\cdot]:=\mathbf{E}[\cdot \mid \Theta_t]$. There are three steps for the proofs.

    \textbf{Step 1}: we show that in each round $t\geq \max\{\sqrt{m},4\}$, we have
\begin{equation}\label{eq:Bing1}
\begin{aligned}
    &\mathbf{E}_{\Theta_t} \left[\sum\limits_{a \in A_t^*} r_{a} -  \sum\limits_{a \in A_t} {w}_{a,t}\right]  \\
    &\leq  2\Phi(-\sqrt{4/\gamma})^{-1} \\
    &\qquad\cdot\mathbf{E}_{\Theta_t}\left[  \left(  \sum\limits_{a\in A_t} {w}_{a,t} - \mathbf{E}_{\Theta_t}\left[\sum\limits_{a\in A_t} {w}_{a,t} \right]  \right)^+ \right].
    \end{aligned}
\end{equation}
    
    \textbf{Step 2}: we further bound the expectation term in the RHS of (\ref{eq:Bing1}) as follows.
\begin{equation}\label{eq:22222}
\begin{aligned}
    &\mathbf{E}_{\Theta_t}\left[  \left(  \sum\limits_{a\in A_t} w_{a,t} - \mathbf{E}_{\Theta_t}\left[\sum\limits_{a\in A_t} {w}_{a,t} \right]  \right)^+  \right] \\
    &\leq   \mathbf{E}_{\Theta_t} \left[ \left| \sum\limits_{a\in A_t} w_{a,t} - \sum\limits_{a\in A_t} \tilde{w}_{a,t} \right|  \right].      
\end{aligned}
\end{equation}

\textbf{Step 3}: summing over $T$, we show that (\ref{eq:22222}) is upper bounded as follows.
\begin{equation}
    \mathbf{E}\left[ \sum_{t=1}^T\left| \sum\limits_{a\in A_t} w_{a,t} - \sum\limits_{a\in A_t} \tilde{w}_{a,t} \right|  \right] \leq 4m\ln T \sqrt{3\gamma NT}
\end{equation}

Combining these three steps, we have
\begin{equation}
\begin{aligned}
    &\mathbf{E}  \left[  \sum_{t=\max\{\sqrt{m},4\}}^T  \left( \sum_{a \in A_t^*} r_{a} - \sum_{a \in A_t}w_{a,t}  \right)  \right] \\
    &\leq 2\Phi(-\sqrt{4/\gamma})^{-1}\mathbf{E}\left[ \sum_{t=\max\{\sqrt{m},4\}}^T  \left|\sum\limits_{a\in A_t} w_{a,t} - \sum\limits_{a\in A_t} \tilde{w}_{a,t} \right| \right]\\
    &\leq 2\Phi(-\sqrt{4/\gamma})^{-1} \mathbf{E}\left[ \sum_{t=1}^T\left| \sum\limits_{a\in A_t} w_{a,t} - \sum\limits_{a\in A_t} \tilde{w}_{a,t} \right|\right]\\
    &\leq 8\sqrt{3\gamma}\Phi(-\sqrt{4/\gamma})^{-1} m\ln T \sqrt{NT}.
\end{aligned}
\end{equation}

Now, we give the details for these three steps.

Let $\alpha := \mathbf{E}_{\Theta_t} \left[\sum_{a \in A_t^*} r_{a} -  \sum_{a \in A_t} w_{a,t}\right]$.

\paragraph{Step 1 proof.} 

If $\alpha = \mathbf{E}_{\Theta_t} \left[\sum_{a \in A_t^*} r_{a} -  \sum_{a \in A_t} w_{a,t}\right] \leq 0$, the proof is trivial as the RHS of (\ref{eq:Bing1}) is non-negative. Note that $2\Phi(-\sqrt{4/\gamma})^{-1} < + \infty$
 
For the case where $\alpha > 0$
, we view $\left(  \sum_{a\in A_t} w_{a,t} - \mathbf{E}_{\Theta_t}\left[\sum_{a\in A_t} w_{a,t} \right]  \right)^+ \ge 0$ as a non-negative random variable and use
Markov's inequality. We have 
\begin{equation}
\begin{aligned}
 &\mathbf{E}_{\Theta_t}\left[  \left(  \sum_{a\in A_t} w_{a,t} - \mathbf{E}_{\Theta_t}\left[\sum_{a\in A_t} w_{a,t} \right]  \right)^+ \right] \\
 &\ge  \alpha \Pr_{\Theta_t} \left(  \sum_{a\in A_t} w_{a,t} - \mathbf{E}_{\Theta_t}\left[\sum_{a\in A_t} w_{a,t} \right]     \geq \alpha  \right),    
\end{aligned}
\end{equation}
which gives
\begin{equation}
\resizebox{1\hsize}{!}{$
\begin{aligned}
   \alpha & \leq  \frac{\mathbf{E}_{\Theta_t}\left[  \left(  \sum\limits_{a\in A_t} w_{a,t} - \mathbf{E}_{\Theta_t}\left[\sum\limits_{a\in A_t} w_{a,t} \right]  \right)^+ \right]}{\Pr_{\Theta_t}\left(  \sum\limits_{a\in A_t} w_{a,t} - \mathbf{E}_{\Theta_t}\left[\sum\limits_{a\in A_t} w_{a,t}  \right]    \geq \alpha  \right)} \\ 
   & =  \frac{\mathbf{E}_{\Theta_t}\left[  \left(  \sum\limits_{a\in A_t} w_{a,t} - \mathbf{E}_{\Theta_t}\left[\sum\limits_{a\in A_t} w_{a,t} \right]  \right)^+ \right]}{\Pr_{\Theta_t}\left(  \sum\limits_{a\in A_t} w_{a,t} - \mathbf{E}_{\Theta_t}\left[\sum\limits_{a\in A_t} w_{a,t}  \right]    \geq \mathbf{E}_{\Theta_t} \left[\sum\limits_{a \in A_t^*} r_{a} -  \sum\limits_{a \in A_t}w_{a,t}\right]  \right)}\\
   & = \frac{\mathbf{E}_{\Theta_t}\left[  \left(  \sum\limits_{a\in A_t} w_{a,t} - \mathbf{E}_{\Theta_t}\left[\sum\limits_{a\in A_t} w_{a,t} \right]  \right)^+ \right]}{\Pr_{\Theta_t}\left(  \sum\limits_{a\in A_t} w_{a,t}   \geq  \sum\limits_{a \in A_t^*} r_{a} \right)} \overset{\rm (a)}{\le}  \frac{\mathbf{E}_{\Theta_t}\left[  \left(  \sum\limits_{a\in A_t} w_{a,t} - \mathbf{E}_{\Theta_t}\left[\sum\limits_{a\in A_t} w_{a,t} \right]  \right)^+ \right]}{\Pr_{\Theta_t}\left(  \sum\limits_{a\in A^*_t} w_{a,t}     \geq  \sum\limits_{a \in A_t^*} r_{a} \right)}\\
     & \overset{\rm (b)}{\le}   2 \Phi(-\sqrt{4/\gamma})^{-1} \cdot \mathbf{E}_{\Theta_t}\left[  \left(  \sum\limits_{a\in A_t} w_{a,t} - \mathbf{E}_{\Theta_t}\left[\sum\limits_{a\in A_t} w_{a,t} \right]  \right)^+ \right],
\end{aligned}
$}
\end{equation}
where step (a) is due to that $A_t$ is the optimal super arm, and thus, we have
$\sum_{a \in A_t^*}w_{a,t} \leq \sum_{a \in A_t} w_{a,t}$ and step (b) uses the result shown in Lemma~\ref{lm:anti1}.

\paragraph{Step 2 proof.}
Recall that $w_{a,t}$ and $\tilde{w}_{a,t}$  are i.i.d. according to $\mathcal{N} \left(\hat{r}_{a, n_{a,t}}, \frac{m \gamma \ln t}{n_{a,t}+1} \right)$, and $A_t$ is the optimal super arm based on $\Theta_t$ and $\bf{w}$. We have $ \mathbf{E}_{\Theta_t}\left[\sum_{a\in A_t} w_{a,t} \right] = \mathbf{E}_{\Theta_t}\left[\max_{A \in \Theta_t}\sum_{a\in A} w_{a,t} \right] = \mathbf{E}_{\Theta_t}\left[\max_{A \in \Theta_t}\sum_{a\in A} \tilde{w}_{a,t} \right] \ge \mathbf{E}_{\Theta_t}\left[\sum_{a\in A_t} \tilde{w}_{a,t} \mid A_t \right] = \mathbf{E}_{\Theta_t}\left[\sum_{a\in A_t} \tilde{w}_{a,t} \mid A_t , \bf{w}\right]$. Then, we have

\begin{equation}
\begin{aligned}
    &\mathbf{E}_{\Theta_t}\left[  \left(  \sum_{a\in A_t} w_{a,t} - \mathbf{E}_{\Theta_t}\left[\sum_{a\in A_t} w_{a,t} \right]  \right)^+  \right] \\
    & \le  \mathbf{E}_{\Theta_t}\left[  \left(  \sum_{a\in A_t} w_{a,t} - \mathbf{E}_{\Theta_t}\left[\sum_{a\in A_t} \tilde{w}_{a,t} \mid A_t \right]  \right)^+  \right]\\
     &=  \mathbf{E}_{\Theta_t}\left[  \left(  \sum_{a\in A_t} w_{a,t} - \mathbf{E}_{\Theta_t}\left[\sum_{a\in A_t} \tilde{w}_{a,t} \mid A_t, \bf{w} \right]  \right)^+  \right]\\
       & =  \mathbf{E}_{\Theta_t}\left[  \left(  \mathbf{E}_{\Theta_t}\left[\left(\sum_{a\in A_t} w_{a,t} - \sum_{a\in A_t} \tilde{w}_{a,t}\right) \mid A_t, \bf{w} \right]  \right)^+  \right]\\
          &\le \mathbf{E}_{\Theta_t}\left[  \left|  \mathbf{E}_{\Theta_t}\left[\left(\sum_{a\in A_t} w_{a,t} - \sum_{a\in A_t} \tilde{w}_{a,t}\right) \mid A_t, \bf{w} \right]  \right|  \right]\\
    &\leq  \mathbf{E}_{\Theta_t}\left[ \mathbf{E}_{\Theta_t} \left[\left|\sum_{a\in A_t} w_{a,t} - \sum_{a\in A_t} \tilde{w}_{a,t}     \right| \mid A_t, \bf{w} \right]\right]\\
    &=  \mathbf{E}_{\Theta_t} \left[ \left| \sum_{a\in A_t} w_{a,t} - \sum_{a\in A_t} \tilde{w}_{a,t} \right|  \right],\\
\end{aligned}
\end{equation}
where the last inequality is due to Jensen's inequality.

\paragraph{Step 3 proof.}
Since $w_{a,t} -  \tilde{w}_{a,t} \sim \mathcal{N}\left(0, \frac{2\gamma m \ln t}{n_{a,t}+1}\right)$, we can express $w_{a,t} -  \tilde{w}_{a,t}$ as $\sqrt{2}\zeta_{a,t} \delta_{a,t}$, where $\zeta_{a,t} \sim \mathcal{N}(0,1)$ and $\delta_{a,t} = \sqrt{\frac{\gamma m \ln t}{n_{a,t}+1}}$. Thus, we have
\begin{equation}
\begin{aligned}
      &\mathbf{E}\left[ \sum_{t=1}^T\left| \sum\limits_{a\in A_t} w_{a,t} - \sum\limits_{a\in A_t} \tilde{w}_{a,t} \right|  \right] \\
      &\leq \sqrt{2}\mathbf{E}\left[\sum_{t=1}^T \sum_{a \in A_t} |\zeta_{a,t}\delta_{a,t}|\right] \\
      & \overset{\rm (a)}{\leq}\sqrt{2}\mathbf{E}\left[\max_{t \in [T], a \in [N]} |\zeta_{a,t}| \sum_{t=1}\sum_{a \in A_t} |\delta_{a,t}|\right] \\
      & = \sqrt{2} \mathbf{E}\left[\max_{t \in [T], a \in [N]} |\zeta_{a,t}| \sum_{t=1}^T\sum_{a \in A_t} \sqrt{\frac{\gamma m \ln t}{n_{a,t}+1}}\right]\\
      & \overset{\rm (b)}{\leq} 2m\sqrt{2 \gamma NT \ln T} \mathbf{E}\left[\max_{t \in [T], a \in [N]} \zeta_{a,t}\right] \\
      & \overset{\rm (c)}{\leq} 2m\sqrt{2 \gamma NT \ln T} \cdot \sqrt{6\ln T}\\
      & \leq 4m\ln T \sqrt{3\gamma NT}.
\end{aligned}    
\end{equation}
where step (a) is due to Hölder's inequality. Step (b) is due to Lemma~\ref{lm:sumofinverse} such that $\sum_{t=1}^T\sum_{a \in A_t} \sqrt{\frac{1}{n_{a,t}+1}} \leq 2\sqrt{mNT}$. Step (c) is due to the maximal inequality for Gaussian variables~(Fact~\ref{lm:maximal4subgaussian}) such that $\mathbf{E}\left[\max_{t \in [T], a \in [N]} \zeta_{a,t}\right]\leq \sqrt{2\ln 2NT}\leq \sqrt{6\ln T}$ because $2\leq N \leq T$.
\end{proof}

\subsection{Proof of Lemma~\ref{lm:newdevi}}\label{sec:newdevi}
\begin{lemma}\label{lm:newdevi}
     Let $\mathcal{E}_t:=\left\{\forall{a \in [N]}, \hat{r}_{a,n_{a,t}} - r_a \leq \sqrt{\frac{3\ln Nt}{n_{a,t} + 1}}\right\}$.
     In CTS-G, the regret of the  deviation part is
     $$
     \begin{aligned}
         \mathbf{E} \left[ \sum_{t=1}^T \left( \sum_{a \in A_t} w_{a,t} - \sum_{a \in A_t} r_{a} \right)\mathbf{1}[\mathcal{E}_t] \right]    &\leq 2m\ln T \sqrt{6\gamma NT} \\
         &+ 2\sqrt{6mNT\ln T}.
     \end{aligned}$$
\end{lemma}
\begin{proof}
We can do decomposition as follows.
\begin{equation}\label{eq:lm12222}
\begin{aligned}
    \mathbf{E}&\left[\sum_{t=1}^T \left( \sum_{a\in A_t} w_{a,t} - \sum_{a\in A_t} {r}_{a} \right) \mathbf{1}[\mathcal{E}_t] \right]   \\
    = \mathbf{E}& \left[ \sum_{t=1}^T \left( \sum_{a\in A_t} w_{a,t} -\sum_{a\in A_t} \hat{r}_{a, n_{a,t}} \right.\right.\\
    &\left.\left.+ \sum_{a\in A_t} \hat{r}_{a, n_{a,t}} - \sum_{a\in A_t} r_a \right)\mathbf{1}[\mathcal{E}_t]  \right] \\
    = \mathbf{E}& \left[  \sum_{t=1}^T  \left(\sum_{a\in A_t} w_{a,t} -\sum_{a\in A_t} \hat{r}_{a, n_{a,t}}  \right)  \mathbf{1}[\mathcal{E}_t]\right] \\
    &\quad+  \mathbf{E} \left[\sum_{t=1}^T  \left( \sum_{a\in A_t} \hat{r}_{a, n_{a,t}} -\sum_{a\in A_t} r_{a}  \right)\mathbf{1}[\mathcal{E}_t]  \right]\\
    \overset{\rm (a)}{\leq}  \mathbf{E}& \left[  \sum_{t=1}^T  \sum_{a\in A_t}\left( w_{a,t} - \hat{r}_{a, n_{a,t}}  \right) \right] + \mathbf{E}\left[\sum_{t=1}^T\sum_{a \in A_t}  \sqrt{\frac{6\ln T}{n_{a,t}+1}}\right]\\
    \overset{\rm (b)}{\leq}   \mathbf{E}& \left[  \sum_{t=1}^T  \sum_{a\in A_t}\left( w_{a,t} - \hat{r}_{a, n_{a,t}}  \right) \right] + 2\sqrt{6 mNT\ln T},
\end{aligned}
\end{equation}
where step (a) is because event $\mathcal{E}_t$ is true and $\ln NT \leq 2\ln T$ because of $N\leq T$, and step (b) is due to Lemma~\ref{lm:sumofinverse} such that $\sum_{t=1}^T\sum_{a \in A_t}  \sqrt{\frac{1}{n_{a,t}+1}} \leq 2\sqrt{mNT}$.

We can represent each $w_{a,t}-\hat{r}_{a,n_{a,t}}$ by $\zeta_{a,t} \delta_{a,t}$, where $\zeta_{a,t} \sim \mathcal{N}(0,1)$ and $\delta_{a,t} = \sqrt{\frac{\gamma m \ln t}{n_{a,t}+1}}$. Then, we can bound the first term on the RHS of the above equation as follows:
\begin{equation}\label{eq:gapwr1}
\begin{aligned}
        &\mathbf{E}\left[\sum_{t=1}^T\sum_{a\in A_t}\left( w_{a,t} - \hat{r}_{a, n_{a,t}}  \right) \right]\\
        &\leq \mathbf{E}\left[\sum_{t=1}^T\sum_{a\in A_t}\zeta_{a,t} \delta_{a,t}  \right]\\
        & \overset{\rm (a)}{\leq} \mathbf{E}\left[\max_{t \in [T], a \in [N]}  |\zeta_{a,t}| \cdot  \sum_{t=1}^T\sum_{a\in A_t} |\delta_{a,t}| \right] \\
        & = \mathbf{E}\left[\max_{t \in [T], a \in [N]}  |\zeta_{a,t}|  \cdot \sum_{t=1}^T\sum_{a\in A_t} \sqrt{\frac{\gamma m \ln t}{n_{a,t}+1}}\right],\\
\end{aligned}
\end{equation}
where (a) is due to Hölder's inequality. By invoking Lemma~\ref{lm:sumofinverse} again, we have that
\begin{equation}\label{eq:gapwr2}
\begin{aligned}
     \sum_{t=1}^T\sum_{a\in A_t} \sqrt{\frac{\gamma m \ln t}{n_{a,t}+1}} &\leq \sqrt{\gamma m \ln T} \sum_{t=1}^T\sum_{a\in A_t} \sqrt{\frac{1}{n_{a,t}+1}} \\
     &\leq 2m \sqrt{\gamma NT \ln T }.
\end{aligned}
\end{equation}
Then, using the maximal inequality~(Fact~\ref{lm:maximal4subgaussian}), we have $\mathbf{E}\left[\max_{t \in [T], a \in [N]} |\zeta_{a,t}|\right]  \leq \sqrt{2\ln 2NT}\leq \sqrt{6\ln T}$, where the last inequality is due to that $2\leq N\leq T$.  Thus, we have 
\begin{equation}\label{eq:3333}
    \mathbf{E}\left[\sum_{t=1}^T\sum_{a\in A_t}\left(w_{a,t} - \hat{r}_{a,n_{a,t}} \right) \right] \leq 2m \ln T \sqrt{6\gamma NT}.
\end{equation}
Finally,  by substituting (\ref{eq:3333}) into (\ref{eq:lm12222}), we complete the proof.
\end{proof}

\subsection{Proof of Lemma~\ref{lm:concerntration}}\label{sec:concerntration}

\begin{lemma}\label{lm:concerntration}
The probability that event $\overline{\mathcal{E}_t}$ to happen satisfies that
    $$
        \sum_{t=1}^T \Pr(\overline{\mathcal{E}_t}) \leq \frac{\pi^2}{3}.
    $$
\end{lemma}
\begin{proof}

By a union bound and Hoeffding's inequality, we have that
\begin{equation}
\begin{aligned}
      &\sum_{t = 1}^T \Pr\left(\exists {a \in [N]}: |r_a - \hat{r}_{a, n_{a,t}}| > \sqrt{\frac{ 3 \ln Nt}{n _{a,t} + 1}}  \right) \\
    \leq \  &\sum_{t = 1}^T\sum_{a \in [N]}\sum_{s=0}^{t-1} \Pr\left(|\hat{r}_{a,s} - r_a| > \sqrt{\frac{3\ln Nt}{s + 1}} \right)\\
     = \  & \sum_{a \in [N]} \sum_{t = 1}^T \left( \Pr\left(r_a > \sqrt{{3\ln Nt}} \right) \right.\\
     &\left.\quad+  \sum_{s=1}^{t-1} \Pr\left(|\hat{r}_{a,s} - r_a| > \sqrt{\frac{3\ln Nt}{s + 1}} \right) \right)\\
     \overset{\rm (a)}{\leq} \  & \sum_{a \in [N]} \left( 0 + \sum_{t = 1}^T\sum_{s=1}^{t-1} \Pr\left(|\hat{r}_{a,s} - r_a| > \sqrt{\frac{3\ln Nt}{2s}} \right) \right)\\
    \leq  \  &N\sum_{t=1}^{\infty} \sum_{s=1}^{t-1} \frac{2}{(Nt)^3} = \frac{\pi^2}{3N^2},
\end{aligned}
\end{equation}
where step (a) is due to $r_a \in [0,1], \forall a \in [N]$ and $3\ln Nt > 1$ because $N \geq 2$, and that $s+1 \leq 2s$ for any $s \geq 1$. 

\end{proof}

\subsection{Proof of Lemma~\ref{lm:anti1}}

\begin{lemma}\label{lm:anti1}
    In each round $t \geq \max\{\sqrt{m},4\}$,  given any $\Theta_t$, we have 
    \begin{equation}
      \frac{1}{ \Pr_{\Theta_t}\left(  \sum_{a\in A^*_t} w_{a,t}     \geq  \sum_{a \in A_t^*} r_{a} \right)}\leq 2\Phi\left(-\sqrt{{4}/{\gamma}}\right)^{-1},
    \end{equation}
    where $\Phi(\cdot)$ is the cdf of the standard Gaussian distribution.   

\end{lemma}

\begin{proof}
Given $\Theta_t$, $A_t^*$ is determined.
Define  $\mathcal{H}_t:= \left\{\sum_{a \in A^*} r_a - \hat{r}_{a, n_{a,t}} \leq \sqrt{\sum_{a \in A^*}\frac{4\ln t}{n_{a,t}+1}}\right\}$.
Since $t \geq \max\{\sqrt{m},4\}$, we have that

\begin{equation}\label{eq:prht}
    \begin{aligned}
        \Pr_{\Theta_t}\left( \mathcal{H}_t  \right) & \ge  1-  \sum\limits_{a \in A_t^*} \sum\limits_{s_a = 0}^{t-1} \Pr_{\Theta_t}\left( |r_a - \hat{r}_{a, s_a}| \ge \sqrt{\frac{4\ln t}{s_a+1}}  \right)  \\
     &  =  1 - \sum\limits_{a \in A_t^*} \sum\limits_{s_a = 1}^{t-1} \Pr_{\Theta_t}\left( |r_a - \hat{r}_{a, s_a}| \ge \sqrt{\frac{4\ln t}{s_a+1}}  \right)  \\
       &  \ge   1 - \sum\limits_{a \in A_t^*} \sum\limits_{s_a = 1}^{t-1} \Pr_{\Theta_t}\left( |r_a - \hat{r}_{a, s_a}| \ge \sqrt{\frac{4\ln t}{2s_a}}  \right)  \\
       & \ge  1-  m t \cdot 2 \cdot  e^{-2  \cdot s_a \cdot  4 \ln t/(2s_a)} \\
       & =  1 - \frac{2mt}{t^4}\ \geq  1- \frac{2}{t}  \ge  0.5\quad.
    \end{aligned}
\end{equation}

We have
\begin{equation}
\resizebox{1\hsize}{!}{$
\begin{aligned}
    &\Pr_{\Theta_t}\left(  \sum_{a\in A^*_t} w_{a,t}     \geq  \sum_{a \in A_t^*} r_{a}  \right) \\
    & \ge \Pr_{\Theta_t}\left(  \sum_{a\in A^*_t} w_{a,t}     \geq  \sum_{a \in A_t^*} r_{a} , \mathcal{H}_t \right) \\
    & =    \Pr_{\Theta_t}\left( \sum_{a\in A^*_t} w_{a,t}  - \hat{r}_{a, n_{a,t}}    \geq  \sum_{a \in A_t^*} r_{a} - \hat{r}_{a, n_{a,t}} , \mathcal{H}_t \right) \\
    & = \Pr_{\Theta_t}(\mathcal{H}_t) \cdot  \Pr_{\Theta_t}\left( \sum_{a\in A^*_t} w_{a,t}  - \hat{r}_{a, n_{a,t}}    \geq  \sum_{a \in A_t^*} r_{a} - \hat{r}_{a, n_{a,t}} \mid \mathcal{H}_t \right) \\
    & \overset{\rm (a)}{\geq} 0.5 \cdot 
 \Pr_{\Theta_t}\left( \sum_{a\in A^*_t} w_{a,t}  - \hat{r}_{a, n_{a,t}}    \geq  \sum_{a \in A_t^*} \sqrt{\frac{ 4\ln t}{n_{a,t} + 1}} \right)\\
    &\overset{\rm (b)}{\geq} 0.5 \cdot \Pr_{\Theta_t}\left( \sum_{a\in A^*_t} w_{a,t}  - \hat{r}_{a, n_{a,t}}\geq  \sqrt{m  \sum_{a \in A_t^*}\frac{ 4\ln t}{n_{a,t}+1}}  \right) \\
     & = 0.5 \cdot \Phi\left(-\sqrt{{4}/{\gamma}}\right),
\end{aligned}
$}
\end{equation}
where step (a) is due to (\ref{eq:prht}) and the fact that event $\mathcal{E}_t$ is true. Step (b) uses the
 Cauchy–Schwarz inequality, i.e., we have $\sum_{a \in A_t^*} \sqrt{\frac{ 4\ln t}{n_{a,t} + 1}} \le \sqrt{m \cdot \sum_{a \in A_t^*} \frac{ 4\ln t}{n_{a,t} + 1}}  $. The last equality is due to the standardization of Gaussian distribution.
\end{proof}

\subsection{Proof of Lemma~\ref{lm:sumofinverse}}\label{sec:sumofinverse}
\begin{lemma}\label{lm:sumofinverse}
 We have $\sum_{t=1}^T \sum_{a \in A_t} \sqrt{ \frac{1}{n_{a,t}+1}} \leq 2\sqrt{mNT}$.
\end{lemma}
\begin{proof}
Note that the LHS of the above inequality is a random variable. We provide an upper bound for this random variable.

Recall $n_{a, t}:=\sum_{\tau=1}^{t-1} \mathbf{1}\left[a \in A_\tau\right]$ is the number of times that arm $a$ has been played at the beginning of round $t$.
Let $\tau_a(n)$ denote the round for arm $a$ to be played for the $n$-th time, and thus $n_{a, \tau_{a}(n)} = n -1$.
\begin{equation}
\begin{aligned}
    &\sum_{t=1}^T \sum_{a \in A_t} \sqrt{ \frac{1}{n_{a,t}+1}} \\
    &=  \sum_{t=1}^T \sum_{a \in [N]} \sqrt{ \frac{1}{n_{a,t}+1}} \mathbf{1}[a\in A_t]\\
    & \overset{\rm (a)}{=} \sum_{a \in [N]} \sum_{n=1}^{n_{a,T+1}} \sum_{t = \tau_a(n)}^{\tau_a(n+1)-1} \sqrt{ \frac{1}{n_{a,t}+1}} \mathbf{1}[a \in A_t]  \\
    & \overset{\rm (b)}{=}   \sum_{a \in [N]} \sum_{n=1}^{n_{a,T+1}}  \sqrt{ \frac{1}{n}} \leq   \sum_{a \in [N]} \int_{0}^{n_{a,T+1}}  \sqrt{ \frac{1}{n}} dn  \\
    & = 2 \sum_{a \in [N]} \sqrt{n_{a,T+1}} \overset{\rm (c)}{\leq}  2  \sqrt{N \sum_{a \in [N]} n_{a,T+1}} \\
    & \overset{\rm (d)}{=}  2\sqrt{mNT},
\end{aligned}
\end{equation}
where step (a) partitions all $T$ rounds into multiple intervals based on the arrivals of observations from arm $a$. Step (b) uses the fact that $\sum_{t = \tau_a(n)}^{\tau_a(n+1)-1} \mathbf{1}[a \in A_t] \cdot \sqrt{ \frac{1}{n_{a,t}+1}} = \sqrt{\frac{1}{n - 1 + 1}} = \sqrt{\frac{1}{n}}$, because $n_{a,\tau_a(n)} = n - 1$ and $\mathbf{1}[a \in A_t] = 0$ for all $t \in \left\{ \tau_a(n)+1, \dotsc, \tau_a(n+1)-1\right\}$. Step (c) uses Cauchy-Schwarz inequality. Step (d) uses the fact that $\sum_{a \in [N]} n_{a,T+1}  \le mT$.
\end{proof}

\subsection{Proof of the Upper Bound of CTS-G}\label{pf:upperctsg}
\begin{proof}[Upper Bound Proof of Theorem~\ref{thm:upperCTS-G}]


Denote by $\mathcal{E}_t:= \left\{\forall{a \in [N]}:|r_a - \hat{r}_{a, n_{a,t}}|\leq \sqrt{\frac{3\ln N t}{n_{a,t}+1}}\right\}$ the high-probability event that the empirical mean reward is close to the true mean reward for arm $a$, and by $\overline{\mathcal{E}_t}$ the complementary event of $\mathcal{E}_t$. 

Let $t^\prime = \max\{\sqrt{m},4\}$. We first decompose the regret as follows:
\begin{equation}\label{eq:decompose}
\begin{aligned}
         \mathcal{R}(T) &= \sum\limits_{t=1}^{t^\prime-1} \mathbf{E} \left[  \sum\limits_{a \in A_t^*} r_{a} -   \sum\limits_{a \in A_t} r_{a} \right] \\
         &\quad+ \sum\limits_{t=t^\prime}^T \mathbf{E} \left[  \sum\limits_{a \in A_t^*} r_{a} -   \sum\limits_{a \in A_t} r_{a} \right] \\
          &\overset{\rm (a)}{\leq}   m\max\{\sqrt{m},4\} \\
          &\qquad+   \mathbf{E} \left[ \sum\limits_{t=t^\prime}^T \left(\sum\limits_{a \in A_t^*} r_{a} -   \sum\limits_{a \in A_t} r_{a}\right) \mathbf{1}[\mathcal{E}_t] \right] \\
          &\qquad+  \mathbf{E} \left[ \sum\limits_{t=t^\prime}^T \left(\sum\limits_{a \in A_t^*} r_{a} -   \sum\limits_{a \in A_t} r_{a}\right) \mathbf{1}[\overline{\mathcal{E}_t}] \right]\\
         &\overset{\rm (b)}{\leq} m\max\{\sqrt{m},4\} \\
         &\qquad+  \mathbf{E} \left[ \sum\limits_{t=t^\prime}^T \left(\sum\limits_{a \in A_t^*} r_{a} -   \sum\limits_{a \in A_t}w_{a,t} \right.\right.\\
         &\qquad\quad\left.\left.+  \sum\limits_{a \in A_t}w_{a,t}- \sum\limits_{a \in A_t} r_{a}\right) \mathbf{1}[\mathcal{E}_t] \right] + m\frac{\pi^2}{3N^2}\\
         &\leq \underbrace{\sum_{t=t^\prime}^T \mathbf{E}  \left[   \left( \sum_{a \in A_t^*} r_{a} - \sum_{a \in A_t}w_{a,t}  \right)  \right] }_{=:I_1,~\text{optimism part}} \\
         &\qquad+  \underbrace{\sum_{t=t^\prime}^T  \mathbf{E} \left[ \sum_{a \in A_t}  \left( w_{a,t}  -  r_{a}  \right) \mathbf{1}[\mathcal{E}_t]\right]}_{=:I_2,~\text{deviation part}} \\
         &\qquad\quad+ m\max\{\sqrt{m},4\} + \frac{\pi^2}{3},
\end{aligned}
\end{equation}
where step (a) is due to the fact that $ \sum_{a \in A_t^*} r_{a} - \sum_{a \in A_t} r_{a} \leq m$ by the definition of $r_a$ and $m$ and step (b) is due to Lemma~\ref{lm:concerntration}.

Now, invoking Lemma~\ref{lm:1} with proofs in Appendix~\ref{sec:lm:1}, we have term $I_1$ bounded as follows:
\begin{equation}
    I_1 \leq 8\sqrt{3\gamma}\Phi(-\sqrt{4/\gamma})^{-1}m\ln T \sqrt{NT},
\end{equation}
and $I_2$ can be bounded by using Lemma~\ref{lm:newdevi} with proofs in Appendix~\ref{sec:newdevi}:
\begin{equation}
    I_2 \leq 2m\ln T \sqrt{6\gamma NT} + 2\sqrt{6mNT\ln T}.
\end{equation}

Thus, we have that 
\begin{equation}\label{eq:thm1main}
\begin{aligned}
\mathcal{R}(T) & \leq \left(2\sqrt{6\gamma} + 8\sqrt{3\gamma}\Phi(-\sqrt{4/\gamma})^{-1}\right)m\ln T \sqrt{NT} \\
&\quad +2\sqrt{6mNT\ln T} + m\left(\max\{\sqrt{m}, 4\} + \frac{\pi^2}{3}\right). \\
\end{aligned}
\end{equation}
Using numerical optimization methods searching from $\gamma = 0.0001$ to $\gamma = 100$, we can find that when $\gamma = 6.4$, the coefficient for the first item can achieve a minimum value of $175.74$.
    
\end{proof}

\subsection{Proof of the Lower Bound of CTS-G}\label{pf:lower1}

\begin{proof}[Proof]
      We construct the following problem instance. We consider a top-$m$ problem with $N$ base arms, i.e., the only combinatorial constraint is that any  $m$  out of  $N$ base arms can be played in each round. Let the optimal super arm be $G \subset [N]$, and $|G| = m \ge 1$. We set $N \geq 400m$ and $T$ is large enough such that $T > \frac{16}{25} N\ln T$. For ease of presentation, we set exploration rate $\gamma = 1$, and the proof can be generalized to any $\gamma > 0$.

Let $\Delta := \frac{4}{5}\sqrt{\frac{N\ln T}{T}}$. For each base arm $a \in [N]$, we consider a deterministic reward setting, 
    defined as follow:
    \begin{equation}
        r_a= \begin{cases} \Delta, &  a \in G, \\ 0, &  a \in[N] \backslash G. \end{cases}
    \end{equation}
Every time if a suboptimal base arm is played,   the amount of regret is $\Delta$. Let $k_t \in \{0,1,\dotsc, m \}$ denote the  number of suboptimal arms played in each round $t$. Note that $k_t$ is random variable and the distribution for $k_t$ is determined by the history information $\mathcal{F}_{t-1}$. 
Note that the  regret suffered in  round $t$ is exactly  $k_t \Delta$. Thus, we can lower bound $\mathcal{R}_T$ by lower bounding the expected times that  subotpimal arms have been played by the end of round $T$, expressed as
    \begin{equation}
\begin{array}{lll}
\mathbf{E}\left[\mathcal{R}_T\right] &=& \Delta \mathbf{E}\left[{\sum}_{t=1}^T k_t \right].
\end{array}
    \end{equation}
    We further define $K_{t-1}:=\sum_{ s = 1}^{t-1}k_s $ as the total number of times that suboptimal arms have been played by the end of round $t-1$. 
    Now, we consider two mutually exclusive and exhaustive cases: $\exists t \in (T/2, T]: K_{t-1} > \frac{1}{2} m (t-H)$ and $\forall t \in (T/2, T]: K_{t-1} \le \frac{1}{2} m (t-H)$, where $H := \left\lceil \frac{64m\ln T}{\Delta^2}\right\rceil$. By our choice of $N$, we have that $H \leq \frac{100m T}{N} \leq T/4$.

    \

    \noindent \textbf{Case 1: $\exists t \in (T/2, T]: K_{t-1} > \frac{1}{2} m (t-H)$}. Let $t_0 \in (T/2, T] $ be the first round such that $K_{t_0} > \frac{1}{2} m (t_0 + 1-H)$.
      In this case, we lower bound  the total number of times that suboptimal base arms have been played by the end of round $t_0$. We have
  \begin{equation}
      \sum_{t=1}^T k_t  \ge \sum_{t=1}^{t_0} k_t  
     =  K_{t_0} 
      >
     \frac{1}{2} m(t_0 + 1 -H)  \ge  \frac{1}{2}  m(T/2 - T/4),
        \end{equation}
    where the last inequality uses the fact that $t_0 > T/2$ and $H \le T/4$. 
Taking an expectation at both sides gives
    \begin{equation}
        \mathbf{E}[\mathcal{R}_T] \geq \Delta \mathbf{E}[K_{t_0}]  = \Omega(m\sqrt{NT\ln T}),
    \end{equation}
   which concludes the proof. 

\
    \noindent \textbf{Case 2: $\forall t \in (T/2, T]:  K_{t-1} \leq \frac{1}{2}m(t-H)$}.   Let 
    $\tau := \Delta +\frac{\Delta}{4}$. We claim  that, with at least a constant probability $p_0 := 0.044$, there are at least $0.1 m$ optimal base arms  with learning models 
    $w_{a,t} < \tau$ and there are at least $0.1 m$ sub-optimal base arms with  learning models 
    $w_{a,t} \ge \tau$. Combining these two results gives that, with at least a constant probability $p_0^2$,  the number of played sub-optimal base arms $k_t$ is at least $0.1 m$ in round $t$. 
    Formally, let $Y_t^*$ denote the event that the number of optimal base arms with $w_{a,t} < \tau$ is at least $0.1m$, i.e., 
$\sum_{a \in G} \mathbf{1}[ w_{a,t} < \tau] \ge 0.1 m $. Similarly, let $Y_t$ denote the event that the number of sub-optimal base arms with $w_{a,t} \ge \tau$ is at least $0.1 m$, i.e., 
$\sum_{a \in [N] \setminus G} \mathbf{1}[ w_{a,t} \ge \tau] \ge 0.1 m $. We lower bound the regret as 
        \begin{equation}
\begin{array}{lll}
\mathbf{E}\left[\mathcal{R}_T\right] &=& \Delta {\sum}_{t=1}^T \mathbf{E}\left[k_t \right] \\
& \ge & \Delta {\sum}_{t=1}^T \Pr\left( Y_t^*, Y_t \right) \cdot \mathbf{E}\left[k_t \mid Y_t^*, Y_t\right]\\
& \ge & \Delta {\sum}_{t=1}^T \Pr\left( Y_t^*, Y_t \right) \cdot0.1 m\\
& = & \Delta {\sum}_{t=1}^T  \mathbf{E}[ \Pr( Y_t^*, Y_t \mid \mathcal{F}_{t-1})] \cdot 0.1m \\
& = & \Delta {\sum}_{t=1}^T  \mathbf{E}[ \Pr( Y_t^*  \mid \mathcal{F}_{t-1})  \cdot \Pr( Y_t  \mid \mathcal{F}_{t-1}) ] \cdot 0.1m,
\end{array}
    \end{equation}
where the last step uses the fact that, given $\mathcal{F}_{t-1}$, random variables $Y_t^*$ and $Y_t$ are independent. 
To complete the proof, we use our novel technical lemma stated below.    \begin{lemma}\label{lemma: temp2}
    For any $t > T/2$ such that $K_{t-1} \le 0.5m(t-H)$, 
 we have $ \Pr( Y_t^*  \mid \mathcal{F}_{t-1} )  \ge p_0$ and  $ \Pr( Y_t  \mid \mathcal{F}_{t-1} )  \ge p_0$. 
    \end{lemma}
With Lemma~\ref{lemma: temp2} in hand, we have
 \begin{equation}
\begin{array}{l}
\mathbf{E}\left[\mathcal{R}_T\right] \ge \Delta  \cdot T \cdot  p_0 \cdot p_0 \cdot 0.1 m
 =  \Omega \left(m\sqrt{NT \ln T} \right),
\end{array}
    \end{equation}
which concludes the proof.
\end{proof}

\begin{proof}[Proof of Lemma~\ref{lemma: temp2}]

To prove $ \Pr( Y_t^*  \mid \mathcal{F}_{t-1})  \ge p_0$, 
we let $\bar{G}_{t}:=\{a \in G: n_{a,t}\geq H\}$ denote the subset of optimal base arms that have been played at least $H$ times. Intuitively, since each optimal base arm $a \in \bar{G}_{t}$ has been observed enough, it is likely to have $w_{a,t} < \tau$.  From Lemma~\ref{lemma: optimal 2}, we have $\left| \bar{G}_{t}\right| \ge 0.5m$.
Let $X_t^*:= \sum_{a \in G} \mathbf{1}[w_{a,t} < \tau]$. Now, we construct a lower bound for $\mathbf{E}\left[X_t^* \mid \mathcal{F}_{t-1} \right]$. We have 
\begin{equation}\label{eq: temp}
    \begin{array}{lll}
         \mathbf{E}\left[X_t^* \mid \mathcal{F}_{t-1} \right] & = & \mathbf{E}\left[\sum_{a \in G} \mathbf{1}[w_{a,t} < \tau]  \mid \mathcal{F}_{t-1} \right] \\
       &  \ge & \mathbf{E}\left[\sum_{a \in \bar{G}_t} \mathbf{1}[w_{a,t} < \tau]  \mid \mathcal{F}_{t-1} \right] \\
          &  = & \sum_{a \in \bar{G}_t}\mathbf{E}\left[ \mathbf{1}[w_{a,t} < \tau]  \mid \mathcal{F}_{t-1} \right] \\
             &  = & \sum_{a \in \bar{G}_t}\Pr\left( w_{a,t} < \tau \mid \mathcal{F}_{t-1}  \right) \\
             &  \ge^{(a)} & \sum_{a \in \bar{G}_t}0.5 \\
            & \ge& 0.5m \cdot 0.5  \\
             & = & 0.25m,
    \end{array}
\end{equation}
where step (a) uses Lemma~\ref{lemma: optimal}.
With equation (\ref{eq: temp}) in hand, 
from Chernoff bound shown in (\ref{eq:chernoff}), we have 
\begin{equation}\label{eq: temp3}
\begin{array}{lll}
\Pr(Y_t^* \mid \mathcal{F}_{t-1})
&=& 1 - \Pr(X_t^* \leq (1-0.6)\cdot 0.25m\mid \mathcal{F}_{t-1})\\ 
& \geq & 1 - \Pr(X_t^* \leq (1-0.6)\mathbf{E}[X_t^* \mid \mathcal{F}_{t-1}] \mid \mathcal{F}_{t-1}) \\
&\geq & 1- \exp\left(-\frac{(0.6)^2\cdot \mathbf{E}[X_t^* \mid \mathcal{F}_{t-1}]}{2}\right)\\
&\geq & 1- \exp\left(-0.045\right)\\
& \geq & 0.044 =p_0,
\end{array}\end{equation}
where the second last inequality uses the fact that $\mathbf{E}[X_t^* \mid \mathcal{F}_{t-1}] \ge 0.25m \ge 0.25$.\footnote{An alternative way  to lower bound $\Pr( Y_t^*  \mid \mathcal{F}_{t-1})$ is to use  Cantelli's inequality shown in (\ref{eq:cantelli}). Let $\sigma_*^2 \le m^2$ denote the variance of $X_t^*$ given $\mathcal{F}_{t-1}$. Then, we have
\begin{equation}
    \begin{array}{ll}
      &   \Pr( Y_t^*  \mid \mathcal{F}_{t-1}) \\ 
             = & \Pr \left(X_t^*- \mathbf{E} \left[ X_t^* \mid \mathcal{F}_{t-1}\right] \ge 0.1 m - \mathbf{E} \left[ X_t^* \mid \mathcal{F}_{t-1}\right] \mid \mathcal{F}_{t-1}\right) \\
             \ge & \Pr \left(X_t^*- \mathbf{E} \left[ X_t^* \mid \mathcal{F}_{t-1}\right] \ge 0.1 m - 0.25 m \mid \mathcal{F}_{t-1}\right) \\
             \ge & 1- \frac{1}{1 + (0.15m)^2/\sigma_*^2} \\
             \ge & 1- \frac{1}{1 + 0.15^2}\\
             \ge & 0.022.
    \end{array}
\end{equation}Note that (\ref{eq: temp 222}) cannot reuse this argument as $m^2$ is not an upper bound of the variance of $X_t$ given $\mathcal{F}_{t-1}$.}


To prove $ \Pr( Y_t  \mid \mathcal{F}_{t-1} )  \ge p_0$, we let $L:=  \frac{16 m\ln (T/2)}{25\Delta^2} -1 $. Then, let set $B_{t}:=\{a \in [N]\setminus G: n_{a,t} \leq L\}$. Intuitively, since each sub-optimal base arm $a \in B_{t}$ has not been observed enough, over-exploration happens and $w_{a,t}$ has a chance to be greater than or equal to $\tau$. From Lemma~\ref{lemma: sub-optimal 2}, we have $\left|B_{t} \right| \ge 2m$.
Let $X_t:= \sum_{a \in [N] \setminus G} \mathbf{1}[w_{a,t} \ge \tau]$. Now, we construct a lower bound for $\mathbf{E}\left[X_t \mid \mathcal{F}_{t-1} \right]$. We have
\begin{equation}\label{eq: temp4}
    \begin{array}{lll}
         \mathbf{E}\left[X_t \mid \mathcal{F}_{t-1} \right] & = & \mathbf{E}\left[\sum_{a \in [N] \setminus G} \mathbf{1}[w_{a,t} \ge \tau]  \mid \mathcal{F}_{t-1} \right] \\
       &  \ge & \mathbf{E}\left[\sum_{a \in B_t} \mathbf{1}[w_{a,t} \ge \tau]  \mid \mathcal{F}_{t-1} \right] \\
          &  = & \sum_{a \in B_t}\mathbf{E}\left[ \mathbf{1}[w_{a,t} \ge \tau]  \mid \mathcal{F}_{t-1} \right] \\
             &  = & \sum_{a \in B_t}\Pr\left( w_{a,t} \ge \tau \mid \mathcal{F}_{t-1}  \right) \\
             &  \ge^{(a)} & \sum_{a \in B_t} 0.15  \\
            & \ge  & 2m \cdot 0.15 \\
            & > & 0.25m,
        
    \end{array}
\end{equation}
where step (a) uses Lemma~\ref{lemma: temp 111}.
Then, we reuse (\ref{eq: temp3}) and have
\begin{equation}\label{eq: temp 222}
    \Pr\left\{Y_t \mid \mathcal{F}_{t-1} \right\} \ge 0.044 = p_0.
\end{equation}\end{proof}




\begin{lemma}\label{lemma: optimal 2}
   For any $t > T/2$ such that $K_{t-1} \le 0.5m(t-H)$, we have $|\bar{G}_{t}| \geq \frac{m}{2}$.
\end{lemma}
\begin{proof}
Let   $G_{t}:= \{a \in G: n_{a,t} < H\}$ denote the subset of the optimal base arms that have been played less than $H$ times.
 On one hand, we have 
\begin{equation}\label{eq:goodarmsH}
\begin{aligned}
    \sum_{a \in G} n_{a,t} &\leq |G_{t}| \cdot (H-1) + |\bar{G}_{t}|\cdot(t-1) \\
                           & = (m-|\bar{G}_{t}|)\cdot (H-1) + |\bar{G}_{t}|\cdot (t-1).
\end{aligned}
\end{equation}
On the other hand, we  have
\begin{equation}\label{eq:goodarms}
\begin{aligned}
\sum_{a \in G} n_{a,t} = m(t-1)-K_{t-1} \geq m(t-1) - \frac{1}{2}m(t-H).
\end{aligned} 
\end{equation}
Combining (\ref{eq:goodarms}) and (\ref{eq:goodarmsH}) gives $|\bar{G}_{t}| \geq \frac{m}{2}$.
\end{proof}

\begin{lemma}\label{lemma: optimal}
    For any optimal base arm $a\in \bar{G}_{t}$, we have $\Pr( w_{a,t} < \tau \mid \mathcal{F}_{t-1} ) \geq 0.5$.
\end{lemma}
\begin{proof}
   For any optimal base arm $a\in \bar{G}_{t}$, we let $Z_{a,t} \sim \mathcal{N}(0, \frac{ m \ln t}{n_{a,t} + 1})$. Then, we have
\begin{equation}
\begin{aligned}
   & \Pr(w_{a,t} < \tau  \mid \mathcal{F}_{t-1}  ) \\
    =& \Pr\left(w_{a,t} < \Delta + \frac{\Delta}{4}    \mid  \mathcal{F}_{t-1}         \right)  \\
      =& \Pr\left(\hat{r}_{a,n_{a,t}} + Z_{a,t} < \Delta + \frac{\Delta}{4}    \mid  \mathcal{F}_{t-1}      \right)  \\
       =& \Pr\left(\Delta + Z_{a,t} < \Delta + \frac{\Delta}{4}    \mid  \mathcal{F}_{t-1}       \right)  \\
     =&  \Pr\left(Z_{a,t} < \frac{\Delta}{8} \mid \mathcal{F}_{t-1}\right)\\
     \ge& \Pr(Z_{a,t} \leq 0 \mid \mathcal{F}_{t-1} ) \\
     = &0.5,
\end{aligned}
\end{equation}
which concludes the proof.
\end{proof}

  \begin{lemma}\label{lemma: sub-optimal 2}
  For any $t > T/2$ such that $K_{t-1} \le 0.5m(t-H)$,
      we have $\left|B_t \right| \ge 2m$.
  \end{lemma}
  \begin{proof} Let  $\bar{B}_{t} =  \{a \in [N]\setminus G: n_{a,t} > L \}$. 
      From the fact that  $K_{t-1} = \sum_{a \in [N]\setminus G} n_{a,t} \leq \frac{1}{2} m (t-H)$, we have 
\begin{equation}
    \frac{1}{2} m (t-H)\geq K_{t-1} \geq {\sum}_{a \in \bar{B}_{t}} n_{a,t} \geq |\bar{B}_{t}|\cdot L.
\end{equation}
From  the fact that $|B_{t}| = N-m - |\bar{B}_{t}|$, we have
\begin{equation}\label{eq:numB}
\begin{aligned}
    |B_{t}| & \geq  N-m - \frac{\frac{1}{2} m (t-H)}{L} \\
    &\geq N-m - \frac{0.5 m T}{\frac{16 m \ln (T/2)}{25\Delta^2}-1} \\
    &\ge  2m,
\end{aligned}
\end{equation}
where the last inequality is due to our choice of $N\geq 400m$.
  \end{proof}

  \begin{lemma}\label{lemma: temp 111}
      For any sub-optimal base arm $a \in B_t$, we have $\Pr(w_{a,t} \geq \tau  {\mid \mathcal{F}_{t-1}}) \ge 0.15$.
  \end{lemma}

  \begin{proof}
      Let $Z_{a,t} \sim \mathcal{N}(0, \frac{ m \ln t}{n_{a,t} + 1})$. For each $a \in B_{t}$, we have  \begin{equation}
\begin{aligned}
    \Pr(w_{a,t} \geq \tau  {\mid \mathcal{F}_{t-1}}) &= \Pr\left(\hat{r}_{a,n_{a,t}} + Z_{a,t} \geq \frac{5\Delta}{4}  {\mid \mathcal{F}_{t-1}}\right)  \\
    & = \Pr\left(Z_{a,t} \geq \frac{5\Delta}{4} \mid \mathcal{F}_{t-1} \right)\\   &= \Pr\left(\mathcal{N}(0,1) \geq \sqrt{\frac{n_{a,t}+1}{ m \ln t}} \cdot \frac{5\Delta}{4} \mid \mathcal{F}_{t-1}\right)    \\
    & \geq \Pr\left(\mathcal{N}(0,1) \geq \sqrt{\frac{L+1}{ m \ln (T/2)}} \cdot \frac{5\Delta}{4} \mid \mathcal{F}_{t-1}\right) \\
    & \geq \Pr\left(\mathcal{N}(0,1) \geq 1\right) \\
    & \ge 0.15,
\end{aligned}
\end{equation}
where we use $\mathcal{N}(0,1)$ to denote a random variable distributed according to the standard normal distribution.
  \end{proof}

\section{Proofs for Theorem~\ref{thm:upperCL-SG}}\label{pfs4CL-G}

\subsection{Proof of Lemma~\ref{lm:suminverse2}}\label{sec:anticon}
\begin{lemma}\label{lm:suminverse2}
    In each round $t > \max\{\sqrt{m},4\}$, given any $\Theta_t$,  we have that for CL-SG:
    \begin{equation}
        \frac{1}{ \Pr_{\Theta_t}\left(  \sum\limits_{a\in A^*_t} \bar{r}_{a,t}     \geq  \sum\limits_{a \in A_t^*} r_{a} \right)} \leq 2 \Phi\left(-\sqrt{{4}/{\gamma}}\right)^{-1}.
    \end{equation}
\end{lemma}
\begin{proof}[Proof of Lemma~\ref{lm:suminverse2}]
Given $\Theta_t$, $A_t^*$ is determined. Define  $\mathcal{H}_t:= \left\{\forall{a \in A_t^*}:|r_a - \hat{r}_{a, n_{a,t}}| \leq \sqrt{\frac{4\ln t}{n_{a,t}+1}}\right\}$. We have
\begin{equation}\label{eq:prht2}
    \begin{aligned}
        \Pr_{\Theta_t}\left( \mathcal{H}_t  \right) & \ge 1-  \sum\limits_{a \in A_t^*} \sum\limits_{s_a = 0}^{t-1} \Pr_{\Theta_t}\left( |r_a - \hat{r}_{a, s_a}| \ge \sqrt{\frac{4\ln t}{s_a+1}}  \right)  \\
     &  =  1 - \sum\limits_{a \in A_t^*} \sum\limits_{s_a = 1}^{t-1} \Pr_{\Theta_t}\left( |r_a - \hat{r}_{a, s_a}| \ge \sqrt{\frac{4\ln t}{s_a+1}}  \right)  \\
       &  \ge   1 - \sum\limits_{a \in A_t^*} \sum\limits_{s_a = 1}^{t-1} \Pr_{\Theta_t}\left( |r_a - \hat{r}_{a, s_a}| \ge \sqrt{\frac{4\ln t}{2s_a}}  \right)  \\
       & \ge  1-  m t \cdot 2 \cdot  e^{-2  \cdot s_a \cdot  4 \ln t/(2s_a)} \\
       &=  1 - \frac{2mt}{t^4} \\
       &\geq  1- \frac{2}{t} \ge  0.5 ,
    \end{aligned}
\end{equation}
where the last two inequalities are due to that $t > \max\{\sqrt{m},4\}$. Then, we have
\begin{equation}
\begin{aligned}
    & \Pr_{\Theta_t}\left(  \sum_{a\in A^*_t} \Bar{r}_{a,t}     \geq  \sum_{a \in A_t^*} r_{a}  \right) 
    \ge \Pr_{\Theta_t}\left(  \sum_{a\in A^*_t} \Bar{r}_{a,t}     \geq  \sum_{a \in A_t^*} r_{a} , \mathcal{H}_t \right) \\
    & =    \Pr_{\Theta_t}\left( \sum_{a\in A^*_t} \Bar{r}_{a,t}  - \hat{r}_{a, n_{a,t}}    \geq  \sum_{a \in A_t^*} r_{a} - \hat{r}_{a, n_{a,t}} , \mathcal{H}_t \right) \\
    & = \Pr_{\Theta_t}(\mathcal{H}_t) \cdot  \Pr_{\Theta_t}\left( \sum_{a\in A^*_t} \Bar{r}_{a,t}  - \hat{r}_{a, n_{a,t}}    \geq  \sum_{a \in A_t^*} r_{a} - \hat{r}_{a, n_{a,t}} \mid \mathcal{H}_t \right) \\
    & \overset{\rm (a)}{\geq} 0.5 \cdot 
 \Pr_{\Theta_t}\left( \sum_{a\in A^*_t} w_t \sqrt{\frac{\gamma \ln t}{n_{a,t}+1}}   \geq  \sum_{a \in A_t^*} \sqrt{\frac{ 4\ln t}{n_{a,t} + 1}} \right)\\
    &= 0.5 \cdot \Pr_{\Theta_t}\left( w_t    \geq  \sqrt{4/\gamma} \right) \\
     & = 0.5 \cdot \Phi\left(-\sqrt{{4}/{\gamma}}\right) ,
\end{aligned}
\end{equation}
where step (a) is due to (\ref{eq:prht2}) and the fact that event $\mathcal{E}_t$ is true.
\end{proof}

\subsection{Proof of Lemma~\ref{lm:opt4CL-SG}}
\begin{lemma}\label{lm:opt4CL-SG}
    The optimism part in CL-SG satisfies that
    \begin{equation}
        I_1 \leq 8\sqrt{2\gamma}\Phi(-\sqrt{4/\gamma})^{-1} \ln T \sqrt{mNT}.
    \end{equation}
\end{lemma}
\begin{proof}
\noindent\textbf{Step 1: Gap Conversion.} We first relate the gap between the optimal reward and the estimated reward via an anti-concentration argument. Specifically, we convert the difference $\sum_{a \in A_t^*} r_a - \sum_{a \in A_t} \bar{r}_{a,t}$ to the deviation of $\sum_{a \in A_t} \bar{r}_{a,t}$ from its expectation:
\begin{equation}\label{eq:pf41111}
\resizebox{1\hsize}{!}{$
    \mathbf{E}_{\Theta_t} \left[\sum\limits_{a \in A_t^*} r_{a} -  \sum\limits_{a \in A_t} \bar{r}_{a,t}\right] \leq   2\Phi(-\sqrt{4/\gamma})^{-1}  \cdot \mathbf{E}_{\Theta_t}\left[  \left(  \sum\limits_{a\in A_t} \bar{r}_{a,t} - \mathbf{E}_{\Theta_t}\left[\sum\limits_{a\in A_t} \bar{r}_{a,t} \right]  \right)^+ \right].
    $}
    \end{equation}
    This step relies on a reverse application of Markov's inequality, along with a careful characterization of the anti-concentration behavior of correlated Gaussian variables (see Lemma~\ref{lm:suminverse2} in Appendix~\ref{sec:anticon}).
    
\noindent\textbf{Step 2: Ghost-Sample Analysis.} To remove the activation function in the above bound, we introduce a ghost sample $\tilde{w}_t \sim \mathcal{N}(0,1)$ that is independent of $w_t$ and define $\tilde{r}_{a,t} = \hat{r}_{a,t} + \tilde{w}_t \cdot \sqrt{\frac{\gamma \ln t}{n_{a,t}+1}}$ for all $a$. We then show that:
    \begin{equation}
    \resizebox{1\hsize}{!}{$
\begin{aligned}
    \mathbf{E}_{\Theta_t}\left[  \left(  \sum\limits_{a\in A_t} \bar{r}_{a,t} - \mathbf{E}_{\Theta_t}\left[\sum\limits_{a\in A_t} \bar{r}_{a,t} \right]  \right)^+  \right] 
    &\leq \mathbf{E}_{\Theta_t} \left[ \left| \sum_{a\in A_t} \bar{r}_{a,t} - \sum_{a\in A_t} \tilde{r}_{a,t} \right| \right] .
\end{aligned}
$}
\end{equation}
\noindent\textbf{Step 3: Aggregation Over Time.} Finally, summing over $T$ rounds and applying Hölder's inequality, we obtain:
\begin{equation}
    \mathbf{E} \left[ \sum_{t=1}^T\left| \sum_{a\in A_t} \bar{r}_{a,t} - \sum_{a\in A_t} \tilde{r}_{a,t} \right| \right] \leq 4 \ln T \sqrt{2\gamma m N T}.
\end{equation}
Combining the above steps completes the proof.

Now we give the details of the three steps as follows.

\noindent\textbf{Step 1 proof.} If $\mathbf{E}_{\Theta_t} \left[\sum\limits_{a \in A_t^*} r_{a} -  \sum\limits_{a \in A_t} \bar{r}_{a,t}\right] \leq 0$, the proof is trivial as the RHS in (\ref{eq:pf41111}) is non-negative.

Recall $(\cdot)^{+} := \max \left\{ \cdot, 0 \right\}$.
For the case where $\alpha := \mathbf{E}_{\Theta_t} \left[\sum\limits_{a \in A_t^*} r_{a} -  \sum\limits_{a \in A_t} \bar{r}_{a,t}\right] > 0$, we use
Markov's inequality and have 
\begin{equation}
\begin{aligned}
 &\mathbf{E}_{\Theta_t}\left[  \left(  \sum\limits_{a\in A_t} \bar{r}_{a,t} - \mathbf{E}_{\Theta_t}\left[\sum\limits_{a\in A_t} \bar{r}_{a,t} \right]  \right)^+ \right]\\ &\geq  \alpha \cdot \Pr_{\Theta_t} \left( \left( \sum\limits_{a\in A_t} \bar{r}_{a,t} - \mathbf{E}_{\Theta_t}\left[\sum\limits_{a\in A_t} \bar{r}_{a,t} \right] \right)^+    \geq \alpha  \right) \\
   & \ge \alpha \cdot \Pr_{\Theta_t} \left(  \sum\limits_{a\in A_t} \bar{r}_{a,t} - \mathbf{E}_{\Theta_t}\left[\sum\limits_{a\in A_t} \bar{r}_{a,t} \right]     \geq \alpha  \right),
\end{aligned}
\end{equation}
which together with Lemma~\ref{lm:suminverse2} gives
\begin{equation}
\resizebox{1\hsize}{!}{$
\begin{aligned}
   \alpha 
   &= \mathbf{E}_{\Theta_t} \left[\sum\limits_{a \in A_t^*} r_{a} -  \sum\limits_{a \in A_t} \bar{r}_{a,t}\right]  \\
   &\leq  \frac{\mathbf{E}_{\Theta_t}\left[  \left(  \sum\limits_{a\in A_t} \bar{r}_{a,t} - \mathbf{E}_{\Theta_t}\left[\sum\limits_{a\in A_t} \bar{r}_{a,t} \right]  \right)^+ \right]}{\Pr_{\Theta_t}\left(  \sum\limits_{a\in A_t} \bar{r}_{a,t} - \mathbf{E}_{\Theta_t}\left[\sum\limits_{a\in A_t} \bar{r}_{a,t}  \right]    \geq \mathbf{E}_{\Theta_t} \left[\sum\limits_{a \in A_t^*} r_{a} -  \sum\limits_{a \in A_t}\bar{r}_{a,t}\right]  \right)}\\
   &\le  \frac{\mathbf{E}_{\Theta_t}\left[  \left(  \sum\limits_{a\in A_t} \bar{r}_{a,t} - \mathbf{E}_{\Theta_t}\left[\sum\limits_{a\in A_t} \bar{r}_{a,t} \right]  \right)^+ \right]}{\Pr_{\Theta_t}\left(  \sum\limits_{a\in A^*_t} \bar{r}_{a,t}     \geq  \sum\limits_{a \in A_t^*} r_{a} \right)}\\
     & =  2\Phi(-\sqrt{4/\gamma})^{-1}  \cdot \mathbf{E}_{\Theta_t}\left[  \left(  \sum\limits_{a\in A_t} \bar{r}_{a,t} - \mathbf{E}_{\Theta_t}\left[\sum\limits_{a\in A_t} \bar{r}_{a,t} \right]  \right)^+ \right].
\end{aligned}
$}
\end{equation}

\noindent\textbf{Step 2 proof.}
Since $w_t$ and $\Tilde{w}_t$ are i.i.d.,  we have $ \mathbf{E}_{\Theta_t}\left[\sum\limits_{a\in A_t} \bar{r}_{a,t} \right] = \mathbf{E}_{\Theta_t}\left[\max_{A \in \Theta_t}\sum\limits_{a\in A} \bar{r}_{a,t} \right] = \mathbf{E}_{\Theta_t}\left[\max_{A \in \Theta_t}\sum\limits_{a\in A} \tilde{r}_{a,t} \right] \ge \mathbf{E}_{\Theta_t}\left[\sum\limits_{a\in A_t} \tilde{r}_{a,t} \mid A_t \right] = \mathbf{E}_{\Theta_t}\left[\sum\limits_{a\in A_t} \tilde{r}_{a,t} \mid A_t , w_t\right]$. Then, we have

\begin{equation}
\begin{aligned}
    &\mathbf{E}_{\Theta_t}\left[  \left(  \sum\limits_{a\in A_t} \bar{r}_{a,t} - \mathbf{E}_{\Theta_t}\left[\sum\limits_{a\in A_t} \bar{r}_{a,t} \right]  \right)^+  \right]  \\
    &\le  \mathbf{E}_{\Theta_t}\left[  \left(  \sum\limits_{a\in A_t} \bar{r}_{a,t} - \mathbf{E}_{\Theta_t}\left[\sum\limits_{a\in A_t} \tilde{r}_{a,t} \mid A_t \right]  \right)^+  \right]\\
    & =  \mathbf{E}_{\Theta_t}\left[  \left(  \sum\limits_{a\in A_t} \bar{r}_{a,t} - \mathbf{E}_{\Theta_t}\left[\sum\limits_{a\in A_t} \tilde{r}_{a,t} \mid A_t, w_t \right]  \right)^+  \right]\\
       & =  \mathbf{E}_{\Theta_t}\left[  \left(  \mathbf{E}_{\Theta_t}\left[\left(\sum\limits_{a\in A_t} \bar{r}_{a,t} - \sum\limits_{a\in A_t} \tilde{r}_{a,t}\right) \mid A_t, w_t \right]  \right)^+  \right]\\
         & \le \mathbf{E}_{\Theta_t}\left[  \left|  \mathbf{E}_{\Theta_t}\left[\left(\sum\limits_{a\in A_t} \bar{r}_{a,t} - \sum\limits_{a\in A_t} \tilde{r}_{a,t}\right) \mid A_t, w_t \right]  \right|  \right]\\
    &\leq  \mathbf{E}_{\Theta_t}\left[ \mathbf{E}_{\Theta_t} \left[\left|\sum_{a\in A_t} \bar{r}_{a,t} - \sum_{a\in A_t} \tilde{r}_{a,t}     \right| \mid A_t, w_t \right]\right]\\
    &\leq  \mathbf{E}_{\Theta_t} \left[ \left| \sum_{a\in A_t} \bar{r}_{a,t} - \sum_{a\in A_t} \tilde{r}_{a,t} \right|  \right].\\
\end{aligned}
\end{equation}

\noindent\textbf{Step 3 proof.}
By Hölder's inequality,  we have that
\begin{equation}
\begin{aligned}
    &\mathbf{E}\left[ \sum_{t=1}^T \left| \sum_{a\in A_t} \bar{r}_{a,t} - \sum_{a\in A_t} \tilde{r}_{a,t} \right| \right] \\
    &\leq \mathbf{E}\left[\sum_{t=1}^T  |w_t - \tilde{w}_t|\sum_{a \in A_t}\sqrt{\frac{\gamma \ln t}{n_{a,t}+1}} \right]\\
    &\leq \mathbf{E}\left[\max_{t \in [T]} |w_t - \tilde{w}_t| \sum_{t=1}^T \sum_{a\in A_t} \sqrt{\frac{\gamma \ln t}{n_{a,t}+1}}\right]\\
    &\overset{\rm (a)}{\leq} \mathbf{E}\left[\max_{t \in [T]} |w_t - \tilde{w}_t|\right] \cdot 2\sqrt{\gamma mNT\ln T}\\
    & \overset{\rm (b)}{\leq} 2\sqrt{\ln 2 T} \cdot 2\sqrt{\gamma mNT\ln T}\leq 4\ln T \sqrt{2\gamma mNT},
\end{aligned}
\end{equation}
where step (a) is due to Lemma~\ref{lm:sumofinverse}, and step (b) is due to Fact~\ref{lm:maximal4subgaussian} and $w_t - \tilde{w}_t$ is a Gaussian variable with variance $2$.
\end{proof}

\subsection{Proof of Lemma~\ref{lm:devisation2}}\label{sec:pfd2}

\begin{lemma}\label{lm:devisation2}
    In CL-SG, the regret of the  deviation part is $$
    \resizebox{1\hsize}{!}{$
    \mathbf{E} \left[ \sum\limits_{t=1}^T \left( \sum\limits_{a \in A_t} \bar{r}_{a,t} - \sum\limits_{a \in A_t} r_{a} \right) \mathbf{1}[\mathcal{E}_t] \right]   \leq  4\ln T\sqrt{\gamma mNT} + 2\sqrt{6mNT\ln T}.
    $}$$
\end{lemma}
\begin{proof}[Proof of Lemma~\ref{lm:devisation2}]

Recall that $\Bar{r}_{a,t} = \hat{r}_{a,n_{a,t}} + w_t \sqrt{\frac{\gamma \ln t}{n_{a,t}+1}}$. When $\mathcal{E}_t$ happens, we have that
\begin{equation}\label{eq:devia2}
\resizebox{1.0\hsize}{!}{$
\begin{aligned}
    &\mathbf{E}\left[\sum_{t=1}^T \left( \sum_{a\in A_t} \bar{r}_{a,t} - \sum_{a\in A_t} {r}_{a} \right) \mathbf{1}[\mathcal{E}_t] \right]  \\
    &= \mathbf{E} \left[ \sum_{t=1}^T\sum_{a\in A_t} \left(  \hat{r}_{a,n_{a,t}} + w_t \sqrt{\frac{\gamma \ln t}{n_{a,t}+1}} - \hat{r}_{a,n_{a,t}} + \sqrt{\frac{3\ln Nt}{n_{a,t}+1}}\right)\mathbf{1}[\mathcal{E}_t]  \right] \\
    & \leq {\sqrt{\gamma \ln T}  \mathbf{E}\left[\sum_{t=1}^T w_t \left(\sum_{a\in A_t} \sqrt{\frac{1}{n_{a,t}+ 1}} \right) \right]}  + \sqrt{6\ln T}\mathbf{E}\left[\sum_{t=1}^T\sum_{a \in A_t} \sqrt{\frac{1}{n_{a,t}+1}}\right],
\end{aligned}
$}
\end{equation}
where the last inequality is due to that $N\leq T$.
Regarding the first item in RHS of~(\ref{eq:devia2}), we can apply Hölder's inequality to have that

\begin{equation}
\begin{aligned}
    &\mathbf{E}\left[\sum_{t=1}^T w_t \left(\sum_{a\in A_t} \sqrt{\frac{1}{n_{a,t}+ 1}} \right) \right] \\
    &\leq  \mathbf{E}\left[\max_{1\leq t\leq T} {|w_t|} \cdot \left|\sum_{t=1}^T\sum_{a \in A_t}\sqrt{\frac{1}{n_{a,t}+1}}\right| \right]\\
    &\leq \mathbf{E}\left[\max_{1\leq t\leq T} |w_t| \cdot 2\sqrt{mNT} \right] \leq 4\sqrt{mNT\ln T},
\end{aligned}
\end{equation}
where the second inequality is due to Lemma~\ref{lm:sumofinverse}, and the last inequality is due to the maximal inequality~(Fact~\ref{lm:maximal4subgaussian}) for Gaussian variables such that $\mathbf{E}\left[\max_{1\leq t\leq T} |w_t|\right] \leq \sqrt{2\ln 2T} \leq 2\sqrt{T}$.

Regarding the second term in RHS of~(\ref{eq:devia2}), we can invoke Lemma~\ref{lm:sumofinverse} again to give a bound of $2\sqrt{6mNT\ln T}$.

\end{proof}

\subsection{Proof of the Upper Bound of CL-SG}
\begin{proof}
Let $t^\prime = \max\{\sqrt{m},4\}$. The regret of CL-SG can be decomposed in steps similar to those in (\ref{eq:decompose}) by
\begin{equation}
\begin{aligned}
        \mathcal{R}(T) &\leq   O(mt^\prime) + \underbrace{\sum_{t=t^\prime}^T  \mathbf{E} \left[  \sum_{a \in A_t^*} r_{a} -    \sum_{a \in A_t}  \bar{r}_{a,t}    \right]  }_{=:I_1,~\text{optimism term}} \\ & +  \underbrace{\sum_{t=t^\prime}^T  \mathbf{E} \left[ \mathbf{E}_{\Theta_t} \left[ \sum_{a \in A_t}  \left(  \bar{r}_{a,t}  -  r_{a}  \right) \mathbf{1}[\mathcal{E}_t]\right] \right]}_{=:I_2,~\text{deviation term}} .    
\end{aligned}
\end{equation}
Invoking Lemmas~\ref{lm:opt4CL-SG} and \ref{lm:devisation2}, we can bound $I_1$ and $I_2$ respectively, and the regret is therefore bounded by
\begin{equation}
\begin{aligned}
    \mathcal{R}(T)  \leq & \ \left(4\sqrt{\gamma} + 8\sqrt{2\gamma}\Phi(-\sqrt{4/\gamma})^{-1}\right) \ln T\sqrt{mNT} \\ & + 2\sqrt{6m NT\ln T} + O(mt').
\end{aligned}
\end{equation}
Since $\gamma > 0$ is a constant, we can numerically tune it to minimize the coefficient of the leading term. Through grid search over $\gamma \in [0.0001, 100]$, the optimal value is found to be $\gamma = 4.57$, yielding a minimized coefficient of approximately $144.43$.

\end{proof}

\subsection{Proof of the Lower Bound of CL-SG}\label{thm2:lower}
\begin{proof}[Lower Bound Proof in Theorem~\ref{thm:upperCL-SG}]
The main challenge in the analysis stems from the fact that a single  Gaussian random seed is shared across all base arms for doing exploration. This will  induce dependencies among all the elements in the decision set, i.e., all available super arms. For ease of presentation, we set exploration rate $\gamma = 1$, and the proof can be generalized to any $\gamma > 0$.

We still consider a top-$m$ problem, i.e., the only combinatorial constraint is that any  $m$  out of  $N$ base arms can be played in each round. We set the number of base arms $N = 2 m$. Let set $G \subset [N]$ with size $|G| = m$ be the optimal super arm.  

For all sufficiently large $T$, we define $\Delta := \sqrt{\frac{N\ln T}{10^4mT}}$.
We consider a deterministic reward setting, where the mean reward of each base arm $a \in [N]$ is set as follows:
    \begin{equation}
        r_a= \begin{cases} \Delta, &  a \in G, \\ 0, &  a \in[N] \backslash G.\end{cases}
    \end{equation}
 Let $k_t$ be the total number of sub-optimal base arms played in round $t$. The regret $\mathbf{E}\left[\mathcal{R}_T\right]$ by the end of round $T$ is lower bounded by
\begin{equation}
\begin{array}{l}
\mathbf{E}\left[\mathcal{R}_T\right] = \Delta \sum_{t=1}^T  \mathbf{E}\left[k_t\right]. 
    \end{array}
\end{equation}
Let $G^*_t := \left\{a \in G: n_{a,t} > t-1 -c_1  T \right\}$, where $c_1 = 10^{-6}$, be the set of optimal base arms that have been sufficiently observed by the end of round $t-1$. Let $\bar{G}_t^* = G \setminus G^*_t$.  Define event $B_t^* := \{ |G^*_t| \ge 0.9995m \}$. Now, we lower bound $\mathbf{E}\left[k_t   \right]$ by considering two  mutually exclusive and exhaustive cases: $\exists t \in [T]: B^{*}_{t}$ is false and $\forall t \in [T]: B^{*}_{t}$ is true.

\

  \noindent    \textbf{Case 1: $\exists t \in [T]: B^{*}_{t}$ is false}. Let $t_0$ be the first  round such that $B^{*}_{t_0}$ is false.   We have the total number of times of playing sub-optimal base arms in $[N] \setminus G$ by the end of round $t_0-1$ is  
\begin{equation}
\begin{array}{l}
{\sum}_{s=1}^{t_0-1} k_s = {\sum}_{b \in [N] \setminus G} n_{b,t_0} 
= m \cdot (t_0-1) -  {\sum}_{a \in G} n_{a,t_0}. 
 \end{array}
\end{equation}
Now,  we  upper bound the total number of times of playing optimal base arms in $G$ by the end of round $t_0-1$. We have
  \begin{equation}
  \begin{array}{ll}
   & {\sum}_{a \in G}  n_{a,t_0} \\
    = &   {\sum}_{a \in G_{t_0}^*}  n_{a,t_0} +  {\sum}_{a \in \bar{G}_{t_0}^*} n_{a,t_0} \\
    \le & (t_0-1) \cdot \left|G_{t_0}^* \right| + (m-\left|G_{t_0}^* \right|) \cdot (t_0- 1-c_1  T)  \\
     = & m(t_0-1) - mc_1 T+ \left|G_{t_0}^* \right| c_1 T \\
     \le & m(t_0-1) - 0.0005m  c_1 T.
       \end{array}
  \end{equation}
  Thus, we have the total number of times of playing sub-optimal base arms by the end of round $t_0-1$ is  lower bounded by
\begin{equation}
\begin{array}{lll}
{\sum}_{s=1}^{t_0-1} k_s  
&= & m \cdot (t_0-1) -  {\sum}_{a \in G} n_{a,t_0} \\
  &\ge& m \cdot (t_0-1)  -   (m \cdot (t_0-1) - 0.0005mc_1T) \\
  & \ge &0.0005mc_1T.
 \end{array}
\end{equation}
Note that
 $0.0005mc_1T \Delta =  \Omega(\sqrt{mNT\ln T})$ lower bounds  $\mathcal{R}_T$. Adding an expectation at both side concludes the proof. 

\

  \noindent  \textbf{Case 2: $\forall t \in [T]: B^{*}_{t}$ is true}. Let $\alpha = \frac{255}{256}+ c_1$. Let $0 < l_1 < l_2 $ be two universal constants that will be tuned later. We lower bound $\mathbf{E}\left[\mathcal{R}_T\right]$  as
\begin{equation}
\begin{array}{ll}
& \mathbf{E}\left[\mathcal{R}_T\right] = \Delta \sum_{t=1}^T  \mathbf{E}\left[k_t\right] \\

 \ge &  \Delta \sum_{t=\alpha  T + 1 }^T \Pr \left( w_t \in [l_1,l_2]\right) \cdot \mathbf{E}\left[k_t \mid w_t \in [l_1, l_2] \right].
    \end{array}
\end{equation}
Let $p_0:= \Pr \left( w_t \in [l_1,l_2]\right) $ be  a universal constant, where $w_t \sim \mathcal{N}(0,1)$. To complete the proof, the remaining thing is to  lower bound $\mathbf{E}\left[k_t \mid w_t \in [l_1, l_2] \right]$ when $t \ge \alpha T +1$. For any such $t$, we let  $G_t := \left\{b \in [N] \setminus G: n_{b,t} \le t-1 -c_2  T \right\}$, where $c_2 = \frac{15}{16}$, be the subset of  sub-optimal base arms that have not been observed sufficiently by the end of round $t-1$.

Given $w_t$ is distributed in the interval $[l_1,l_2]$, any optimal base arm  $a \in G_t^*$ has
\begin{equation}
    \begin{array}{lll}
 \bar{r}_{a,t}  &= & \hat{r}_{a,t} + w_t \sqrt{\frac{ \ln t}{n_{a,t}+1}} \\
&\le  & \Delta + l_2 \sqrt{\frac{\ln T}{t-1-c_1 T}
} \\
&\le & \Delta + l_2 \sqrt{\frac{\ln T}{\alpha T-c_1T}} \\
&\le & \Delta + l_2 \sqrt{\frac{1}{\alpha-c_1}} 
\cdot \Delta \sqrt{\frac{10^4m}{N}}\\
&= & \Delta + l_2 \sqrt{\frac{256}{255}} \Delta \sqrt{\frac{10^4m}{N}} \\
&<  &  \Delta + 100 l_2 \Delta.
    \end{array}
\end{equation}
Similarly, any sub-optimal base arm $b \in G_t$ has
\begin{equation}
    \begin{array}{lll}
 \bar{r}_{b,t} & = & \hat{r}_{b,t} + w_t \sqrt{\frac{ \ln t}{n_{b,t}+1}} \\
&\ge  & 0 + l_1 \sqrt{\frac{\ln (\alpha  T)}{t-1-c_2 T + 1}
} \\
&\ge & l_1 \sqrt{\frac{\ln (\alpha  T)}{ T-c_2T}} \\
&\ge & l_1 \sqrt{\frac{0.75 \ln  T}{ T-c_2T}} \\
&= & l_1 \sqrt{\frac{0.75}{1-c_2}}\Delta \sqrt{\frac{10^4m}{N}} \\
&= & l_1 \sqrt{16 \cdot 0.75} \Delta \sqrt{\frac{10^4m}{N}}\\
&= & 100\sqrt{6}l_1 \Delta. 
    \end{array}
\end{equation} 
Tuning $l_1 = \frac{1}{2\sqrt{2}}$ and $l_2 = 0.5$ tells us that  $\bar{r}_{a,t}$  for all $a \in G_t^*$ will be  smaller than $\bar{r}_{b,t}$  for all $b \in G_t$. 
Therefore, given $w_t \in [l_1, l_2]$, there are at least $\min \left\{\left|G_t^* \right|,\left|G_t \right|  \right\}$ sub-optimal base arms will be played in each round $t$, i.e.,   $k_t  \ge \min \left\{\left|G_t^* \right|,\left|G_t \right|  \right\}$. Now, we lower bound $\left|G_t \right|$. One one hand, we have
 \begin{equation}
     \begin{array}{lll}
          {\sum}_{b \in [N] \setminus G}  n_{b,t}  
        &= &  {\sum}_{b \in  G_t}  n_{b,t} + {\sum}_{b \in  \bar{G}_t}  n_{b,t} \\
        &\ge & 0 + (N-m - \left| G_t\right|) \cdot (t-1- c_2  T)\\
      & = &m(t-1- c_2  T) -   \left| G_t\right|  \cdot (t-1- c_2  T).
     \end{array}
 \end{equation}
 On the other hand, we have
\begin{equation}
\begin{array}{lll}
 \sum\limits_{b \in [N] \setminus G}  n_{b,t}  &=& m \cdot (t-1) - \sum\limits_{a \in G} n_{a,t}\\
 & \le & m \cdot (t-1) - 0.9995m\cdot (t - 1-c_1  T) \\
 & = & 0.0005 m(t-1) + 0.9995c_1 m T.
    \end{array}
\end{equation}
From  above two, we have
\begin{equation}
    \begin{array}{ll}
         & m\cdot (t-1- c_2  T) -   \left| G_t\right|  \cdot (t-1- c_2  T) \\
         \le & 0.0005 m(t-1) + 0.9995c_1 m T,
    \end{array}
\end{equation}
which gives
 \begin{equation}
 \begin{array}{lll}
 \left| G_t\right| & \ge &\frac{m\cdot (t-1- c_2  T) - 0.0005 m(t-1) - 0.9995c_1 m T}{(t-1- c_2  T)}\\
& \ge & \frac{0.9995m\alpha T - c_2 m T - 0.9995c_1 mT}{(1-c_2)T} \\
& \ge &m \cdot  \frac{0.9995 \cdot \frac{255}{256} - \frac{15}{16}}{1- \frac{15}{16}} \\
& > & 0.92m.
  \end{array}
 \end{equation}
By plugging in   $k_t  \ge \min \left\{\left|G_t^* \right|,\left|G_t \right|  \right\} \ge 0.92m$, we have 
\begin{equation}
\begin{array}{ll}
& \mathbf{E}\left[\mathcal{R}_T\right] = \Delta \sum_{t=1}^T  \mathbf{E}\left[k_t\right] \\

 \ge &  \Delta \sum_{t=\alpha  T + 1 }^T \Pr \left( w_t \in [l_1,l_2]\right) \cdot \mathbf{E}\left[k_t \mid w_t \in [l_1, l_2] \right] \\
 \ge & \Delta (1-\alpha) T \cdot p_0 \cdot 0.92 m \\
 = & \Omega(\sqrt{mNT\ln T}),
    \end{array}
\end{equation}
which concludes the proof.

\end{proof}
\end{document}